%% file: atomo_arxiv.tex
\documentclass[11pt]{article}

\input commands.tex

\title{\atomsgd: Communication-efficient Learning via Atomic Sparsification}

\newcommand*\samethanks[1][\value{footnote}]{\footnotemark[#1]}
\author[1]{{Hongyi Wang\thanks{These authors contributed equally.}}}
\author[2]{{Scott Sievert\samethanks}}
\author[2]{{Zachary Charles}}
\author[1]{{Shengchao Liu}}
\author[1]{{Stephen Wright}}
\author[2]{{Dimitris Papailiopoulos}}

\affil[1]{Department of Computer Sciences, UW--Madison}
\affil[2]{Department of Electrical and Computer Engineering, UW--Madison}

\begin{document}

\maketitle

\begin{abstract}
Distributed model training suffers from communication overheads due to frequent gradient updates transmitted between compute nodes.
To mitigate these overheads, several studies propose the use of sparsified stochastic gradients.
We argue that these are facets of a general sparsification method that can operate on any possible {\it atomic decomposition}. Notable examples include element-wise, singular value, and Fourier decompositions.
We present \atomsgd{}, a general framework for atomic sparsification of stochastic gradients.
Given a gradient, an atomic decomposition, and a sparsity budget, \atomsgd{} gives a random unbiased sparsification of the atoms minimizing variance.
We show that recent methods such as QSGD and TernGrad are special cases of \atomsgd{} 
and that sparsifiying the singular value decomposition of neural networks gradients, rather than their coordinates, can lead to significantly faster distributed training.
\end{abstract}

\input{atomo_intro.tex}

\input{atomo_prelim.tex}
\input{atomo_theory.tex}
\input{atomo_quantization.tex}
\input{atomo_matrix.tex}
\input{atomo_empirical.tex}
\input{atomo_conclusion.tex}

\section*{Acknowledgement}\label{sec:acknowledge}
This work was supported in part by AWS Cloud Credits for Research from Amazon.

\bibliographystyle{plain}
\bibliography{atomo}

\newpage

\appendix

\input{atomo_appendixA.tex}
\input{atomo_kkt_conditions.tex}
\input{atomo_appendixB.tex}
\input{atomo_appendixC.tex}

\end{document}

%% file: commands.tex
\usepackage[T1]{fontenc}
\usepackage[utf8]{inputenc}
\usepackage{url}            
\usepackage{booktabs}       
\usepackage{amsfonts,fullpage}       
\usepackage{nicefrac}       
\usepackage{microtype} 
\usepackage{wrapfig}
\usepackage{color}
\usepackage{enumitem}
\usepackage{graphicx}
\usepackage{amsmath,amssymb}
\usepackage{amsthm}
\usepackage{bigstrut}
\usepackage{float}
\usepackage{caption}
\usepackage[ruled]{algorithm2e}
\usepackage{mathtools}
\usepackage{subfigure}
\usepackage{booktabs}
\usepackage{authblk}

\newtheorem{theorem}{Theorem}
\newtheorem{definition}{Definition}

\newtheorem{lemma}[theorem]{Lemma}

\DeclareMathOperator*{\argmax}{argmax}
\DeclareMathOperator*{\sign}{sign}

\newcommand{\real}{\mathbb{R}}
\newcommand{\cmp}{\mathbb{C}}
\newcommand{\EE}{\mathbb{E}}
\newcommand{\atomsgd}{\textsc{Atomo}}
\newcommand{\eg}{\textit{e.g.}}
\newcommand{\ie}{\textit{i.e.}}

\newcommand\Tstrut{\rule{0pt}{1.5\normalbaselineskip}}
\newcommand\Bstrut{\rule[-1.0\normalbaselineskip]{0pt}{0pt}} 

\usepackage{color}
\definecolor{myblue}{rgb}{.8, .8, 1}

\usepackage{empheq}

\newlength\mytemplen
\newsavebox\mytempbox

\makeatletter
\newcommand\mybluebox{%
	\@ifnextchar[%]
	{\@mybluebox}%
	{\@mybluebox[0pt]}}

\def\@mybluebox[#1]{%
	\@ifnextchar[%]
	{\@@mybluebox[#1]}%
	{\@@mybluebox[#1][0pt]}}

\def\@@mybluebox[#1][#2]#3{
	\sbox\mytempbox{#3}%
	\mytemplen\ht\mytempbox
	\advance\mytemplen #1\relax
	\ht\mytempbox\mytemplen
	\mytemplen\dp\mytempbox
	\advance\mytemplen #2\relax
	\dp\mytempbox\mytemplen
	\colorbox{myblue}{\hspace{1em}\usebox{\mytempbox}\hspace{1em}}}

\makeatother

\usepackage[most]{tcolorbox}

\newtcbox{\mymath}[1][]{%
	nobeforeafter, math upper, tcbox raise base,
	enhanced, colframe=blue!30!black,
	colback=blue!30, boxrule=1pt,
	#1}

%% file: atomo_intro.tex
% !TeX root = atomo_arxiv.tex

\section{Introduction}\label{sec:intro}

	Distributed computing systems have become vital to the success of modern machine learning systems. Work in parallel and distributed optimization has shown that these systems can obtain massive speed up gains in both convex and non-convex settings \cite{chen2016revisiting,dean2012large,duchi2015asynchronous,jaggi2014communication,liu2015asynchronous, recht2011hogwild}. Several machine learning frameworks such as TensorFlow \cite{abadi2016tensorflow}, MXNet \cite{chen2015mxnet}, and Caffe2 \cite{caffe2}, come with distributed implementations of popular training algorithms, such as mini-batch SGD. However, the empirical speed-up gains offered by distributed training, often fall short of the optimal linear scaling one would hope for. It is now widely acknowledged that communication overheads are the main source of this speedup saturation phenomenon~\cite{dean2012large, seide20141, strom2015scalable, qi17paleo,grubic2018synchronous}.
	 
	Communication bottlenecks are largely attributed to frequent gradient updates transmitted between compute nodes.
	 As the number of parameters in state-of-the-art models scales to hundreds of millions \cite{he2016deep, huang2017densely}, the size of  gradients scales proportionally.
	 These bottlenecks become even more pronounced in the context of federated learning \cite{mcmahan2016communication,konevcny2016federated}, where edge devices (\eg, mobile phones, sensors, etc) perform decentralized training, but suffer from low-bandwidth during up-link.

	To reduce the cost of of communication during distributed model training, a series of recent studies propose communicating low-precision or sparsified versions of the computed gradients during model updates. Partially initiated by a 1-bit implementation of SGD by Microsoft in \cite{seide20141}, a large number of recent studies revisited the idea of low-precision training as a means to reduce communication~\cite{
	 	de2015taming,
	 	alistarh2017qsgd,
	 	zhou2016dorefa,
	 	wen2017terngrad,
	 	de2017understanding,
	 	zhang2017zipml,
	 	rastegari2016xnor,
	 	de2017understanding,
	 	de2018high,
	 	bernstein2018signsgd}. 
	 	Other approaches for low-communication training focus on sparsification of gradients, either by thresholding small entries or by random sampling~\cite{
	 	strom2015scalable,
	 	mania2015perturbed,
	 	leblond2016asaga,
	 	aji2017sparse,
	 	lin2017deep,
	 	chen2017adacomp,
	 	renggli2018sparcml,
	 	tsuzuku2018variance}. 
	 Several approaches, including QSGD and TernGrad, implicitly combine quantization and sparsification to maximize performance gains~\cite{
	 	alistarh2017qsgd,
	 	wen2017terngrad,
	 	konevcny2016federated,
	 	konevcny2016randomized,
	 	suresh2016distributed}, while providing provable guarantees for convergence and performance.
	 We note that quantization methods in the context of gradient based updates have a rich history, dating back to at least as early as the 1970s \cite{gitlin1973design, alexander1987transient,bermudez1996nonlinear}.

	\paragraph{Our Contributions}
	An atomic decomposition represents a vector as a linear combination of simple building blocks in an inner product space. In this work, we show that stochastic gradient sparsification and quantization are facets of a general approach that sparsifies a gradient in any possible atomic decomposition, including its entry-wise or singular value decomposition, its Fourier decomposition, and more. 
	With this in mind, we develop \atomsgd{}, a general framework for atomic sparsification of stochastic gradients. \atomsgd{} sets up and optimally solves a meta-optimization that minimizes the variance of the sparsified gradient, subject to the constraints that it is sparse on the atomic basis, and also is an unbiased estimator of the input. 
	
	We show that 1-bit QSGD and TernGrad are in fact  special cases of \atomsgd{}, and each is  optimal (in terms of variance and sparsity), in different parameter regimes. 
	Then, we argue that for some neural network applications, viewing the gradient as a concatenation of matrices (each corresponding to a layer), and applying atomic sparsification to their SVD is meaningful and well-motivated by the fact that these matrices are ``nearly'' low-rank, \eg, see Fig.~\ref{fig:SVdecay}. We show that \atomsgd{} on the SVD of each layer's gradient, can  lead to less variance, and faster training, for the same communication budget as that of QSGD or TernGrad. We  present extensive experiments showing that using \atomsgd{} with SVD sparsification, can lead to up to $2\times$ faster  training time (including the time to compute the SVD) compared to QSGD, on VGG and ResNet-18, and SVHN and CIFAR-10.

\begin{figure}[h]
	\centering
	\includegraphics[width=0.5\textwidth]{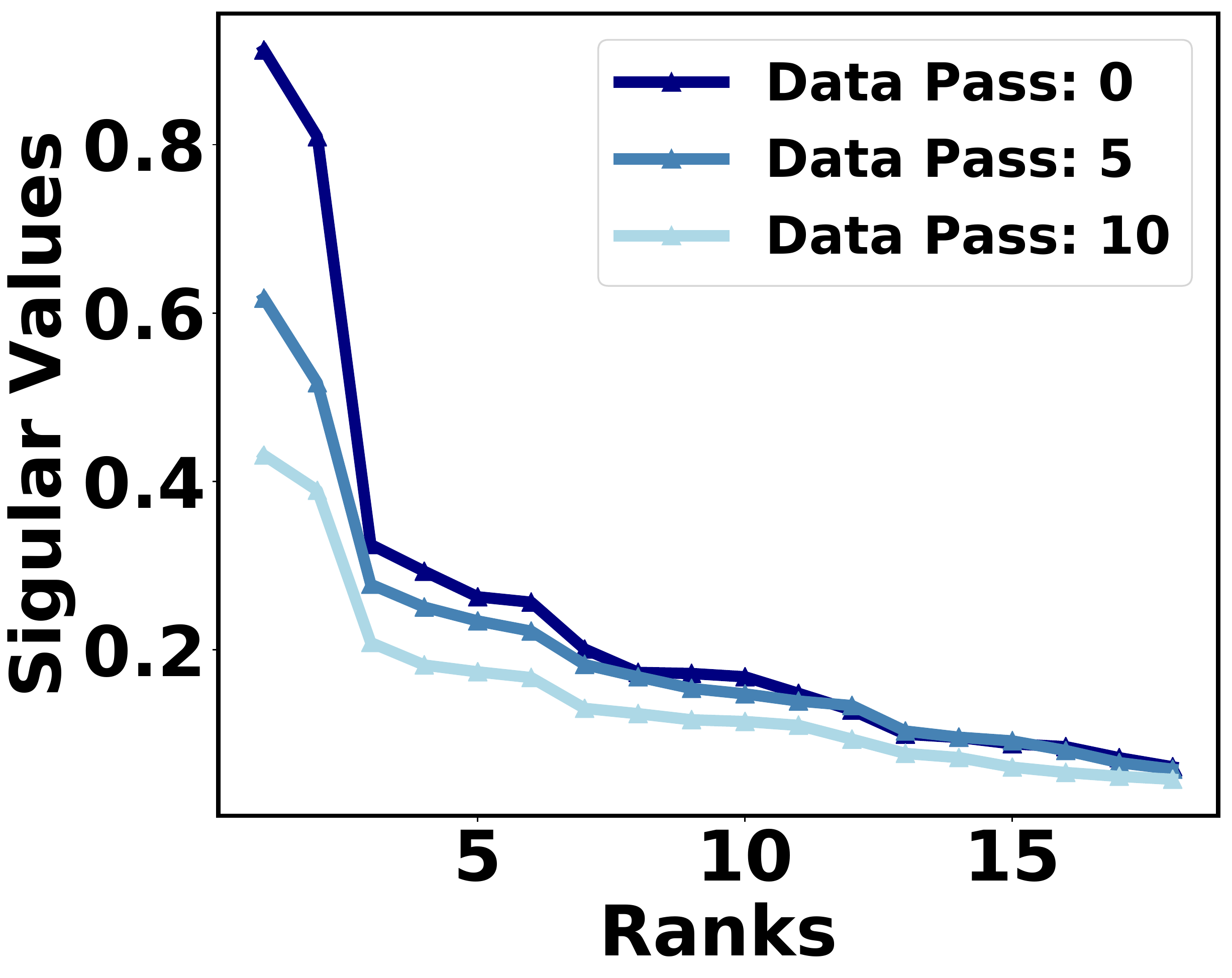}
	\caption{\small The singular values of a convolutional layer's gradient, for ResNet-18 while training on CIFAR-10. The gradient of a layer can be seen as a matrix, once we vectorize and appropriately stack the convolutional filters. For all presented data passes, there is a sharp decay in singular values, with the top 3 standing out.}\label{fig:SVdecay}
\end{figure}

	\paragraph{Relation to Prior Work}

	\atomsgd{} is  closely related to work on communication-efficient distributed mean estimation in \cite{konevcny2016randomized} and \cite{suresh2016distributed}. These works both note, as we do, that  variance (or equivalently the mean squared error) controls important quantities such as convergence, and they seek to find a low-communication vector averaging scheme that minimizes it. Our work differs in two key aspects. First, we derive a closed-form solution to the variance minimization problem for all input gradients. Second, \atomsgd{} applies to any atomic decomposition, which allows us to compare entry-wise against singular value sparsification for matrices. Using this, we derive explicit conditions for which SVD sparsification leads to lower variance for the same sparsity budget.
	
	The idea of viewing gradient sparsification through a meta-optimization lens was also used in \cite{wangni2017gradient}. Our work differs in two key ways. First, \cite{wangni2017gradient} consider the problem of minimizing the sparsity of a gradient for a fixed variance, while we consider the reverse problem, that is, minimizing the variance subject to a sparsity budget. The second more important difference is that while \cite{wangni2017gradient} focuses on entry-wise sparsification, we consider a general problem where we sparsify according to any atomic decomposition. For instance, our approach directly applies to sparsifying the singular values of a matrix, which gives rise to faster training algorithms.
	
	Finally, low-rank factorizations and sketches of the gradients when viewed as matrices were proposed in
	\cite{
		xue2013restructuring,
		sainath2013low,
		jaderberg2014speeding,
		wiesler2014mean,
		konevcny2016federated}; arguably most of these methods (with the exception of \cite{konevcny2016federated}) aimed to address the high flops required during inference by using low-rank models. Though they did not directly aim to reduce communication, this arises as a useful side effect.

%% file: atomo_prelim.tex
% !TeX root = atomo_arxiv.tex

\section{Problem Setup}\label{sec:prelim}

    	In machine learning, we often wish to find a model $w$ minimizing the {\it empirical risk} 
    	\begin{equation}\label{eq:emp_risk}
    	f(w) = \frac{1}{n}\sum_{i=1}^n \ell(w;x_i)\end{equation}
    	where $x_i \in \real^d$ is the $i$-th data point. One way to approximately minimize $f(w)$ is by using stochastic gradient methods that operate as follows:
    	$$w_{k+1} = w_{k}-\gamma \widehat{g}(w_{k})$$
    	where $w_0$ is some initial model, $\gamma$ is the step size, and $\widehat{g}(w)$ is a stochastic gradient of $f(w)$, \ie it is an unbiased estimate of the true gradient $g(w) = \nabla f(w)$. 
    	Mini-batch SGD, one of the most common algorithms for distributed training, computes $\widehat{g}$ as an average of $B$ gradients, each evaluated on randomly sampled data  from the training set. Mini-batch SGD is easily parallelized in the parameter server (PS) setup, where a PS stores the global model, and $P$ compute nodes split the effort of computing the $B$ gradients. Once the PS receives these gradients, it applies them to the model, and sends it back to the compute nodes.
    	
    	To prove convergence bounds for stochastic-gradient based methods, we usually require $\widehat{g}(w)$ to be an unbiased estimator of the full-batch gradient, and to have small second moment $\mathbb{E}\|\widehat{g}(w)\|^2$, as this  controls the speed of convergence.
		To see this, suppose $w^*$ is a critical point of $f$, then we have		
    	\begin{align*}
		\EE[\|w_{k+1}-w^*\|_2^2]
		&= \EE[\|w_{k}-w^*\|_2^2] -\underbrace{\left(2\gamma \langle \nabla f(w_k), w_{k}-w_*\rangle - \gamma^2\mathbb{E}[\|\widehat{g}(w_k)\|^2_2]\right)}_{\text{progress at step $t$}}.
		\end{align*}

		Thus, the progress of the algorithm at a single step is, in expectation, controlled by the term $\EE[\|\widehat{g}(w_k)\|]_2^2$; the smaller it is, the bigger the progress. This is a well-known fact in optimization, and most convergence bounds for stochastic-gradient based methods, including mini-batch, involve upper bounds on $\EE[\|\widehat{g}(w_k)\|]_2^2$ in convex and nonconvex settings
		~\cite{cotter2011better,
			ghadimi2013stochastic,
			recht2011hogwild,
			bubeck2015convex,
			bubeck2015convex,
			de2015global,
			reddi2016stochastic,
			karimi2016linear,
			de2016big,
			yin2018gradient}. 
		In short, recent results on low-communication variants of SGD design unbiased quantized or sparse gradients, and try to minimize their variance~\cite{alistarh2017qsgd,konevcny2016randomized,wangni2017gradient}. Note that when we require the estimate to be unbiased, minimizing the variance is equivalent to minimizing the second moment.
		
		Since variance is a proxy for speed of convergence, in the context of communication-efficient stochastic gradient methods, one can ask: {\it What is the smallest possible variance of an unbiased stochastic gradient that can be represented with $k$ bits?} Note that under the unbiased assumption, minimizing variance is equivalent to minimizing the second moment of the random vector. This meta-optimization can be cast as the following meta-optimization:
			\begin{align}
				&\min_{g} \mathbb{E}\|\widehat{g}(w)\|^2 \nonumber\\
				\text{s.t. }& \mathbb{E}[\widehat{g}(w)] =  g(w)\nonumber \\
				& \widehat{g}(w) \text{ can be expressed with $k$ bits}\nonumber				
			\end{align}			

		Here, the expectation is taken over the randomness of $\widehat{g}$. We are interested in designing a stochastic approximation $\widehat{g}$ that ``solves'' this optimization. However, it seems difficult to design a formal, tractable version of the last constraint. In the next section, we replace this with a simpler constraint that instead requires that $\widehat{g}(w)$ is sparse with respect to a given atomic decomposition.

%% file: atomo_theory.tex
% !TeX root = atomo_arxiv.tex

\section{\atomsgd{}: Atomic Decomposition and Sparsification}\label{sec:theory}

	Let $(V,\langle \cdot, \cdot \rangle)$ be an inner product space over $\real$ and let $\|\cdot\|$ denote the induced norm on $V$. In what follows, you may think of $g$ as a stochastic gradient of the function we wish to optimize. An {\it atomic decomposition} of $g$ is any decomposition of the form $g = \sum_{a \in \mathcal{A}} \lambda_aa$ for some set of atoms $\mathcal{A} \subseteq V$. Intuitively, $\mathcal{A}$ consists of simple building blocks. We will assume that for all $a \in \mathcal{A}$, $\|a\| = 1$, as this can be achieved by a positive rescaling of the $\lambda_a$.
	
	An example of an atomic decomposition is the entry-wise decomposition $g = \sum_i g_ie_i$ where $\{e_i\}_{i=1}^n$ is the standard basis. More generally, any orthonormal basis of $V$ gives rise to a unique atomic decomposition of any $g \in V$. While we focus on finite-dimensional vectors, one could use Fourier and wavelet decompositions in this framework for infinite-dimensional spaces. When considering matrices, the singular value decomposition gives an atomic decomposition in the set of rank-1 matrices. More general atomic decompositions have found uses in a variety of situations, including solving linear inverse problems \cite{chandrasekaran2012convex}.

	We are interested in finding an approximation to $g$ with fewer atoms. Our primary motivation is that this reduces communication costs, as we only need to send atoms with non-zero weights. We can use whichever decomposition is most amenable for sparsification. For instance, if $X$ is a low rank matrix, then its singular value decomposition is naturally sparse, so we can save communication costs by sparsifying its singular value decomposition instead of its entries.

	Suppose $\mathcal{A} = \{a_i\}_{i=1}^n$ and we have an atomic decomposition $g = \sum_{i=1}^n \lambda_ia_i$. We wish to find an unbiased estimator $\widehat{g}$ of $g$ that is sparse in these atoms, and with small variance. Since $\widehat{g}$ is unbiased, minimizing its variance is equivalent to minimizing $\EE[\|\widehat{g}\|^2]$. We use the following estimator:
	\begin{equation}\label{eq:vhat}
	\widehat{g} = \sum_{i=1}^n \dfrac{\lambda_it_i}{p_i}a_i\end{equation}
	where $t_i \sim \text{Bernoulli}(p_i)$, for $0 < p_i \leq 1$. We refer to this sparsification scheme as {\it atomic sparsification}. Note that the $t_i$'s are independent. Recall that we assumed above that $\|a_i\|^2 = 1$ for all $a_i$. We have the following lemma about $\widehat{g}$.

	\begin{lemma}\label{lem:basic}If $g = \sum_{i=1}^n \lambda_ia_i$ is an atomic decomposition then $\EE[\widehat{g}] = g$ and
	$$\EE[\|\widehat{g}\|^2] = \sum_{i=1}^n \dfrac{\lambda_i^2}{p_i} + \sum_{i \neq j} \lambda_i\lambda_j\langle a_i,a_j\rangle.$$\end{lemma}

	Let $\lambda = [\lambda_1,\ldots, \lambda_n]^T,\; p = [p_1,\ldots,p_n]^T$. In order to ensure that this estimator is sparse, we fix some {\it sparsity budget} $s$. That is, we require $\sum_i p_i = s$; note that this a {\it sparsity on average} constraint. We wish to minimize $\EE[\|\widehat{g}\|^2]$ subject to this constraint. By Lemma \ref{lem:basic}, this is equivalent to 
	\begin{equation}\label{eq:opt_prob1}
		\min_{p}~\sum_{i=1}^n \dfrac{\lambda_i^2}{p_i}\quad \text{subject to}~~0 < p_i \leq 1,~~\sum_{i=1}^n p_i = s.
	\end{equation}

	An equivalent form of this optimization problem was previously presented in \cite{konevcny2016randomized} (Section 6.1). The authors considered this problem for entry-wise sparsification and found a closed-form solution for $s \leq \|\lambda\|_1/\|\lambda\|_\infty$. We give a version of their result but extend this to a closed-form solution for all $s$. A similar optimization problem was given in \cite{wangni2017gradient}, which instead minimizes sparsity subject to a variance constraint.

	\vspace{0.5cm}

	\begin{algorithm}[H]
		\SetKwInOut{Input}{Input}
		\SetKwInOut{Output}{Output}
		\Input{$\lambda \in \real^n$ with $|\lambda_1| \geq \ldots |\lambda_n| $; sparsity budget $s$ such that $0 < s \leq n$.}
		\Output{$p \in \real^n$ solving \eqref{eq:opt_prob1}.}
		$i = 1$\;
		\While{$i \leq n$}{
			\eIf{ $|\lambda_i|s \leq \sum_{j=i}^n |\lambda_i|$}{
				\For{$k = i,\ldots, n$}{
				$p_k = |\lambda_k|s\left(\sum_{j=i}^n |\lambda_i|\right)^{-1}$\;
				}
				$i = n+1$\;				
			}{
				$p_i = 1, s = s-1$\;
				$i = i+1$\;
			}
		}
		\caption{\atomsgd{} probabilities}
		\label{alg:opt_sparse}
	\end{algorithm}

	\vspace{0.5cm}

	We will show that the Algorithm~\ref{alg:opt_sparse} produces a probability vector $p \in \real^n$ solving \eqref{eq:opt_prob1} for $0 < s \leq n$. While we show in Appendix \ref{sec:appendix_kkt} that this result can be derived using the KKT conditions, we use an alternative method that focuses on a relaxation of \eqref{eq:opt_prob1} in order to better understand the structure of the problem. This approach has the added benefit of shedding light on what variance is achieved by solving \eqref{eq:opt_prob1}.

	Note that \eqref{eq:opt_prob1} has a non-empty feasible set only for $0 < s \leq n$. Define $f(p) := \sum_{i=1}^n\lambda_i^2/p_i$. To understand how to solve \eqref{eq:opt_prob1}, we first consider the following relaxation:
	\begin{equation}\label{eq:opt_prob2}
		\min_{p}~\sum_{i=1}^n \dfrac{\lambda_i^2}{p_i}\quad \text{subject to}~~ 0 < p_i,~~\sum_{i=1}^n p_i = s.
	\end{equation}

	We have the following lemma about the solutions to \eqref{eq:opt_prob2}, first shown in \cite{konevcny2016randomized}.
	\begin{lemma}[\cite{konevcny2016randomized}]\label{thm:opt_prob2}
		Any feasible vector $p$ to \eqref{eq:opt_prob2} satisfies $f(p) \geq \dfrac{1}{s}\|\lambda\|_1^2$.
		This is achieved iff
		\begin{equation}\label{eq:opt_p}
		p_i = \dfrac{|\lambda_i|s}{\|\lambda\|_1}.\end{equation}
	\end{lemma}

	Lemma \ref{thm:opt_prob2} implies that if we ignore the constraint that $p_i \leq 1$, then the optimal $p$ is achieved by setting $p_i = |\lambda_i|s/\|\lambda\|_1$. If the quantity in the right-hand side is greater than 1, this does not give us an actual probability. This leads to the following definition.

	\begin{definition}\label{def:s-balanced}An atomic decomposition $g = \sum_{i=1}^n \lambda_ia_i$ is $s$-unbalanced at entry $i$ if $|\lambda_i|s > \|\lambda\|_1$.
	\end{definition}

	Fix the atomic decomposition of $g$. If there are no $s$-unbalanced entries then we say that the $g$ is $s$-balanced. We have the following lemma which guarantees that $g$ is $s$-balanced for $s$ not too large.

	\begin{lemma}\label{lem:suff_balanced}
	An atomic decomposition $g = \sum_{i=1}^n \lambda_ia_i$ is $s$-balanced iff $s \leq \|\lambda\|_1/\|\lambda\|_\infty$.\end{lemma}

	Lemma \ref{thm:opt_prob2} gives us the optimal way to sparsify $s$-balanced vectors, since the $p$ that is optimal for \eqref{eq:opt_prob2} is feasible for \eqref{eq:opt_prob1}. Moreover, the iff condition in Lemma \ref{thm:opt_prob2} implies that the optimal assignment of the $p_i$ are between 0 and 1 iff $v$ is $s$-balanced. Suppose now that $g$ is $s$-unbalanced at entry $j$. We cannot assign $p_j$ as in \eqref{eq:opt_p}. We will show that setting $p_j = 1$ is optimal in this setting. This comes from the following lemma.

	\begin{lemma}\label{lem:unbalanced}
		Suppose that $g$ is $s$-unbalanced at entry $j$ and that $q$ is feasible in \eqref{eq:opt_prob1}. Then $\exists p$ that is feasible in \eqref{eq:opt_prob1} such that $f(p) \leq f(q)$ and $p_j = 1$.
	\end{lemma}

	Lemmas \ref{thm:opt_prob2} and \ref{lem:unbalanced} imply the following theorem about solutions to \eqref{eq:opt_prob1}.

	\begin{theorem}\label{thm:all}Suppose we sparsify $g$ as in \eqref{eq:vhat} with sparsity budget $s$.
	\begin{enumerate}
		\item If $g$ is $s$-balanced, then
	$$\EE[\|\widehat{g}\|^2] \geq \frac{1}{s}\|\lambda\|_1^2 + \sum_{i \neq j} \lambda_i\lambda_j\langle a_i,a_j\rangle$$
	with equality if and only if $p_i = |\lambda_i|s/\|\lambda\|_1$.
		\item If $g$ is $s$-unbalanced, then
	$$\EE[\|\widehat{g}\|^2] > \frac{1}{s}\|\lambda\|_1^2 + \sum_{i \neq j} \lambda_i\lambda_j\langle a_i,a_j\rangle$$
	and is minimized by $p$ with $p_j = 1$ where $j = \argmax_{i = 1,\ldots, n} |\lambda_i|$.
	\end{enumerate}
	\end{theorem}

	This theorem implies that Algorithm \ref{alg:opt_sparse} produces a vector $p \in \real^n$ solving \eqref{eq:opt_prob1}. Note that due to the sorting requirement in the input, the algorithm requires $O(n\log n)$ operations. As we discuss in Appendix \ref{sec:appendix_kkt}, we could instead do this in $O(sn)$ operations by, instead of sorting and iterating through the values in order, simply selecting the next unvisited index $i$ maximizing $|\lambda_i|$ and performing the same test/updates. As we show in Appendix \ref{sec:appendix_kkt}, we need to select at most $s$ indices before the if statement in Algorithm \ref{alg:opt_sparse} holds. Whether to sort or do selection depends on the size of $s$ relative to $\log n$.

%% file: atomo_quantization.tex
% !TeX root = atomo_arxiv.tex

\section{Relation to QSGD and TernGrad}\label{sec:comparison}

	In this section, we will discuss how \atomsgd{} is related to two recent quantization schemes, 1-bit QSGD \cite{alistarh2017qsgd} and TernGrad \cite{wen2017terngrad}. We will show that in certain cases, these schemes are versions of the \atomsgd{} for a specific sparsity budget $s$. Both schemes use the entry-wise atomic decomposition.

	\subsection{1-bit QSGD}

		QSGD takes as input $g \in \real^n$ and $b \geq 1$. This $b$ governs the number of quantization buckets. When $b = 1$, this is referred to as 1-bit QSGD. 1-bit QSGD produces a random vector $Q(g)$ defined by
		$$Q(g)_i =  \|g\|_2\sign(g_i)\zeta_i.$$
		Here, the $\zeta_i \sim \text{Bernoulli}(|g_i|/\|g\|_2)$ are independent random variables. A straightforward computation shows that $Q(g)$ can be defined equivalently by
		\begin{equation}\label{eq:qsgd_alt}
		Q(g)_i = \dfrac{g_it_i}{|g_i|/\|g\|_2}\end{equation}
		where $t_i \sim \text{Bernoulli}(|g_i|/\|g\|_2)$. Therefore, 1-bit QSGD exactly uses the atomic sparsification framework in \eqref{eq:vhat} with $p_i = |g_i|/\|g\|_2$. The total sparsity budget is therefore given by
		$$s = \sum_{i=1}^n p_i = \dfrac{\|g\|_1}{\|g\|_2}.$$

		By Lemma \ref{lem:suff_balanced} any $g$ is $s$-balanced for this $s$. Therefore, Theorem \ref{thm:all} implies that the optimal way to assign $p_i$ with this given $s$ is $p_i = |g_i|/\|g\|_2$. Since this agrees with \eqref{eq:qsgd_alt}, this implies that 1-bit QSGD performs variance-optimal entry-wise sparsification for sparsity budget $s = \|g\|_1/\|g\|_2$.

	\subsection{TernGrad}

		Similarly, TernGrad takes as input $g \in \real^n$, and produces a sparsified version $T(g)$ given by
		$$T(g)_i = \|g\|_\infty\sign(g_i)\zeta_i$$
		where $\zeta_i \sim \text{Bernoulli}(|g_i|/\|g\|_\infty)$. A straightforward computation shows that $T(g)$ can be defined equivalently by
		\begin{equation}\label{eq:terngrad_alt}
		T(g)_i = \dfrac{g_it_i}{|g_i|/\|g\|_\infty}\end{equation}
		where $t_i \sim\text{Bernoulli}(|g_i|/\|g\|_\infty)$. Therefore, TernGrad exactly uses the atomic sparsification framework in \eqref{eq:vhat} with $p_i = |g_i|/\|g\|_\infty$. The total sparsity budget is given by
		$$s = \sum_{i=1}^n p_i = \dfrac{\|g\|_1}{\|g\|_\infty}.$$

		By Lemma \ref{lem:suff_balanced}, any $g$ is $s$-balanced for this $s$. Therefore, Theorem \ref{thm:all} implies that the optimal way to assign $p_i$ with this given $s$ is $p_i = |g_i|/\|g\|_\infty$. This agrees with \eqref{eq:terngrad_alt}. Therefore, TernGrad performs variance-optimal entry-wise sparsification for sparsity budget $s = \|g\|_1/\|g\|_\infty$.

	\subsection{$\ell_q$-quantization}

		We can generalize both of these with the following quantization method, which we refer to as $\ell_q${\it-quantization}. Fix $q \in (0,\infty]$. Given $g \in \real^n$, we define the $\ell_q$-quantization of $g$, denoted $L_q(g)$ by
		$$L_q(v)_i = \|g\|_q\sign(g_i)\zeta_i$$
		where $\zeta_i \sim\text{Bernoulli}(|g_i|/\|g\|_q)$. Note that for all $i$, $|g_i| \leq \|g\|_\infty \leq \|g\|_q$, so this does give us a valid probability. We can define $L_q(v)$ equivalently by
		\begin{equation}\label{eq:lq_alt}
		L_q(g)_i = \dfrac{g_it_i}{|g_i|/\|g\|_q}\end{equation}
		where $t_i \sim\text{Bernoulli}(|g_i|/\|g\|_q)$. Therefore, $\ell_q$-quantization exactly uses the atomic sparsification framework in \eqref{eq:vhat} with $p_i = |g_i|/\|g\|_q$. The total sparsity budget is therefore
		$$s = \sum_{i=1}^n p_i = \dfrac{\|g\|_1}{\|g\|_q}.$$

		By Lemma \ref{lem:suff_balanced}, the optimal way to assign $p_i$ with this given $s$ is $p_i = |g_i|/\|g\|_q$. Since this agrees with \eqref{eq:lq_alt}, Theorem \ref{thm:all} implies the following theorem.

		\begin{theorem}\label{thm:lqgrad}
		$\ell_q$-quantization performs atomic sparsification in the standard basis with $p_i = |g_i|/\|g\|_q$. This solves \eqref{eq:opt_prob1} for $s = \|g\|_1/\|g\|_q$ and satisfies $\EE[\|L_q(g)\|_2^2] = \|g\|_1\|g\|_q$.\end{theorem}

		In particular, for $q = 2$ we get 1-bit QSGD while for $q = \infty$, we get TernGrad. This shows strong connections between quantization and sparsification methods. Note that as $q$ increases, the sparsity budget for $\ell_q$-quantization increases while the variance decreases. The choice of whether to set $q = 2$, $q = \infty$, or $q$ to something else is therefore dependent on the total expected sparsity one is willing to tolerate, and can be tuned for different distributed scenarios.

%% file: atomo_matrix.tex
% !TeX root = atomo_arxiv.tex

\section{Spectral-\atomsgd{}: Sparsifying the Singular Value Decomposition} \label{sec:compare_matrix_sparse}

	In this section we compare different methods for matrix sparsification. The first uses \atomsgd{} on the entry-wise decomposition of a matrix, and the second uses \atomsgd{} on the singular value decomposition (SVD) of a matrix. We refer to this second approach as Spectral-\atomsgd{}. We show that under concrete conditions, Spectral-\atomsgd{} incurs less variance than sparsifying entry-wise. We present these conditions and connect them to the equivalence of certain matrix norms.

	\paragraph{Notation}
		For a rank $r$ matrix $X$, denote its singular value decomposition by
		$$X = \sum_{i=1}^r \sigma_i u_iv_i^T.$$
		Let $\sigma = [\sigma_1,\ldots, \sigma_r]^T$. Taking the $\ell_p$ norm of $\sigma$ gives a norm on $X$, referred to as the Schatten $p$-norm. For $p = 1$, we get the spectral norm $\|\cdot\|_*$, for $p = 2$ we get the Frobenius norm $\|\cdot\|_F$, and for $p = \infty$, we get the spectral norm $\|\cdot\|_2$. We define the $\ell_{p,q}$ norm of a matrix $X$ by
		$$\|X\|_{p,q} = \left(\sum_{j=1}^m(\sum_{i=1}^n |X_{i,j}|^p)^{q/p}\right)^{1/q}.$$
		When $p = q = \infty$, we define this to be $\|X\|_{\max}$ where $\|X\|_{\max} := \max_{i,j} |X_{i,j}|$.

	\paragraph{Comparing matrix sparsification methods:}Suppose that $V$ is the vector space of real $n\times m$ matrices. Given $X \in V$, there are two standard atomic decompositions of $X$. The first is the entry-wise decomposition
	$$X = \sum_{i,j}X_{i,j}e_ie_j^T.$$
	The second is the singular value decomposition
	$$X = \sum_{i=1}^r \sigma_iu_iv_i^T.$$
	If $r$ is small, it may be more efficient to communicate the $r(n+m)$ entries of the SVD, rather than the $nm$ entries of the matrix. Let $\widehat{X}$ and $\widehat{X}_\sigma$ denote the random variables in \eqref{eq:vhat} corresponding to the entry-wise decomposition and singular value decomposition of $X$, respectively. We wish to compare these two sparsifications.

	In Table \ref{table:matrix_sparse}, we compare the communication cost and second moment of these two methods. The communication cost is the expected number of non-zero elements (real numbers) that need to be communicated. For $\widehat{X}$, a sparsity budget of $s$ corresponds to $s$ non-zero entries we need to communicate. For $\widehat{X}_\sigma$, a sparsity budget of $s$ gives a communication cost of $s(n+m)$ due to the singular vectors. We compare the optimal second moment from Theorem \ref{thm:all}.

	\begin{table}[ht]
		\caption{Communication cost versus second moment of singular value sparsification and vectorized matrix sparsification of a $n\times m$ matrix.}
		\label{table:matrix_sparse}
		\begin{center}
		\begin{sc}
		\begin{tabular}{lcc}
		\toprule
		 & \begin{tabular}{@{}c@{}}Communication \\ Cost\end{tabular} & \begin{tabular}{@{}c@{}}Second \\ Moment\end{tabular}\\
		\midrule
		Entry-wise & $s$ & $\dfrac{1}{s}\|X\|_{1,1}^2$ \Tstrut \Bstrut\\
		\hline
		SVD  & $s(n+m)$ & $\dfrac{1}{s}\|X\|_*^2$ \Tstrut \Bstrut\\
		\bottomrule
		\end{tabular}
		\end{sc}
		\end{center}
		\vskip -0.1in
	\end{table}	

	To compare the second moment of these two methods under the same communication cost, we $s$ and suppose $X$ is $s$-balanced entry-wise. By Theorem \ref{thm:all} and Lemma \ref{lem:suff_balanced}, the second moment in Table \ref{table:matrix_sparse} is achieved iff
	$$s \leq \dfrac{\|X\|_{1,1}}{\|X\|_{\max}}.$$

	To achieve the same communication cost with $\widehat{X}_{\sigma}$, we take a sparsity budget of $s' = s/(n+m)$. By Theorem \ref{thm:all} and Lemma \ref{lem:suff_balanced}, the second moment in Table \ref{table:matrix_sparse} is achieved iff
	$$s' = \dfrac{s}{n+m} \leq \dfrac{\|X\|_*}{\|X\|_2}.$$
	This leads to the following theorem.

	\begin{theorem}Suppose $X \in \real^{n\times m}$ and
	$$s \leq \min\left\{ \dfrac{\|X\|_{1,1}}{\|X\|_{\max}}, (n+m)\dfrac{\|X\|_*}{\|X\|_2}\right\}.$$
	Then $\widehat{X}_{\sigma}$ with sparsity $s' = s/(n+m)$ incurs the same communication cost as $\widehat{X}$ with sparsity $s$, and $\EE[\|\widehat{X}_{\sigma}\|^2] \leq \EE[\|\widehat{X}\|^2]$ if and only if
	$$(n+m)\|X\|_*^2 \leq \|X\|_{1,1}^2.$$\end{theorem} 

	To better understand this condition, we will make use of the following well-known fact concerning the equivalence of the $\ell_{1,1}$ and spectral norms.
	\begin{lemma}\label{lem:compare_norms}
	For any $n\times m$ matrix $X$ over $\real, \frac{1}{\sqrt{nm}}\|X\|_{1,1} \leq \|X\|_* \leq \|X\|_{1,1}$.\end{lemma}

	For expository purposes, we give a proof of this fact in Appendix \ref{sec:matrix_norms} and show that these bounds are the best possible. In other words, there are matrices for which both of the inequalities are equalities. If the first inequality is tight, then $\EE[\|\widehat{X}_{\sigma}\|^2] \leq \EE[\|\widehat{X}\|^2]$, while if the second is tight then $\EE[\|\widehat{X}_{\sigma}\|^2] \geq \EE[\|\widehat{X}\|^2]$. As we show in the next section, using singular value sparsification can translate in to significantly reduced distributed training time.

%% file: atomo_empirical.tex
% !TeX root = atomo_arxiv.tex

\def\DetailedResultsRowWidth{\textwidth}
\def\DetailedResultsRowGap{\\}
\def\DetailedResultsColumnWidth{0.45\textwidth}
\def\DetailedResultsCurveColumnWidth{0.45\textwidth}
\def\DetailedResultsColumnGap{~}
\def\DetailedResultsIncludegraphicsScale{0.27}

\section{Experiments}
\label{sec:experiments}
We present an empirical study of Spectral-\atomsgd{} and compare it to the recently proposed QSGD \cite{alistarh2017qsgd}, and TernGrad \cite{wen2017terngrad}, on a different neural network models and data sets, under real distributed environments. Our main findings are as follows:
\vspace{-0.2cm}
\begin{itemize}[leftmargin=0.5cm]
	\item We observe that spectral-\atomsgd{} provides a useful alternative to entry-wise sparsification methods, it reduces communication compared to vanilla mini-batch SGD, and can reduce training time compared to QSGD and TernGrad by up to a factor of $2\times$ and $3\times$ respectively. For instance, on VGG11-BN trained on CIFAR-10, spectral-\atomsgd{}  with sparsity budget 3 achieves $3.96\times$ speedup over vanilla SGD, while 4-bit QSGD achieves $1.68\times$ on a cluster of 16, g2.2xlarge instances. Both \atomsgd{} and QSGD greatly outperform TernGrad as well.
	\item We observe that spectral-\atomsgd{} in distributed settings leads to models with negligible accuracy loss when combined with parameter tuning.
\end{itemize}

\paragraph{Implementation and setup} We compare spectral-\atomsgd{}\footnote{code available at:  \url{https://github.com/hwang595/ATOMO}}
with different sparsity budgets to $b$-bit QSGD across a distributed cluster
with a parameter server (PS), implemented in mpi4py \cite{dalcin2011parallel}
and PyTorch \cite{paszke2017automatic} and deployed on multiple types of instances in Amazon EC2 (\eg m5.4xlarge, m5.2xlarge, and g2.2xlarge), both PS and compute nodes are of the same type of instance. The PS implementation is standard, with a
few important modifications. At the most basic level, it receives gradients from the compute nodes and broadcasts the updated model once a batch has been received.

In our experiments, we use data augmentation (random crops, and flips), and tuned the step-size for every different setup as shown in Table~\ref{tab:stepsize} in Appendix~\ref{sec:hyperparam}. Momentum and regularization terms are switched off to make the hyperparamter search tractable and the results more legible.
Tuning the step sizes for this distributed network for three different datasets
and eight different coding schemes can be computationally intensive. As such, we only used small networks so that multiple networks could fit into GPU memory. To emulate the effect of larger networks, we use
synchronous message communication, instead of asynchronous.

Each compute node evaluates
gradients sampled from its partition of data. Gradients are then sparsified
through QSGD or spectral-\atomsgd{}, and then are sent back to the
PS. Note that spectral-\atomsgd{} transmits the weighted singular vectors sampled from the true gradient of a layer. The PS then combines these, and updates the model with the average gradient.
Our entire experimental pipeline is implemented in PyTorch \cite{paszke2017automatic} with
mpi4py \cite{dalcin2011parallel}, and deployed on either g2.2xlarge, m5.2xlarge and
m5.4xlarge instances in Amazon AWS EC2. We conducted our experiments on various
models, datasets, learning tasks, and neural network models as detailed in
Table~\ref{Tab:DataStat}.

\begin{table}[ht]%[h]%[htbp]
	\small
	\centering
	\begin{tabular}{|c|c|c|c|}
		\hline Dataset & CIFAR-10 & CIFAR-100 & SVHN \bigstrut\\
		\hline
		\# Data points & 60,000 & 60,000 & 600,000  \bigstrut\\
		\hline
		Model & ResNet-18 / VGG-11-BN & ResNet-18 & ResNet-18 \bigstrut\\
		\hline
		\# Classes & 10 & 100 & 10 \bigstrut\\
		\hline
		\# Parameters & 11,173k / 9,756k & 11,173k & 11,173k \bigstrut\\
		\hline
	\end{tabular}%
    \vspace{0.5em}
	\caption{The datasets used and their associated learning models and hyper-parameters.}
	\label{Tab:DataStat}%
\end{table}

\begin{figure}[ht]
    \centering
    \includegraphics[width=1\columnwidth]{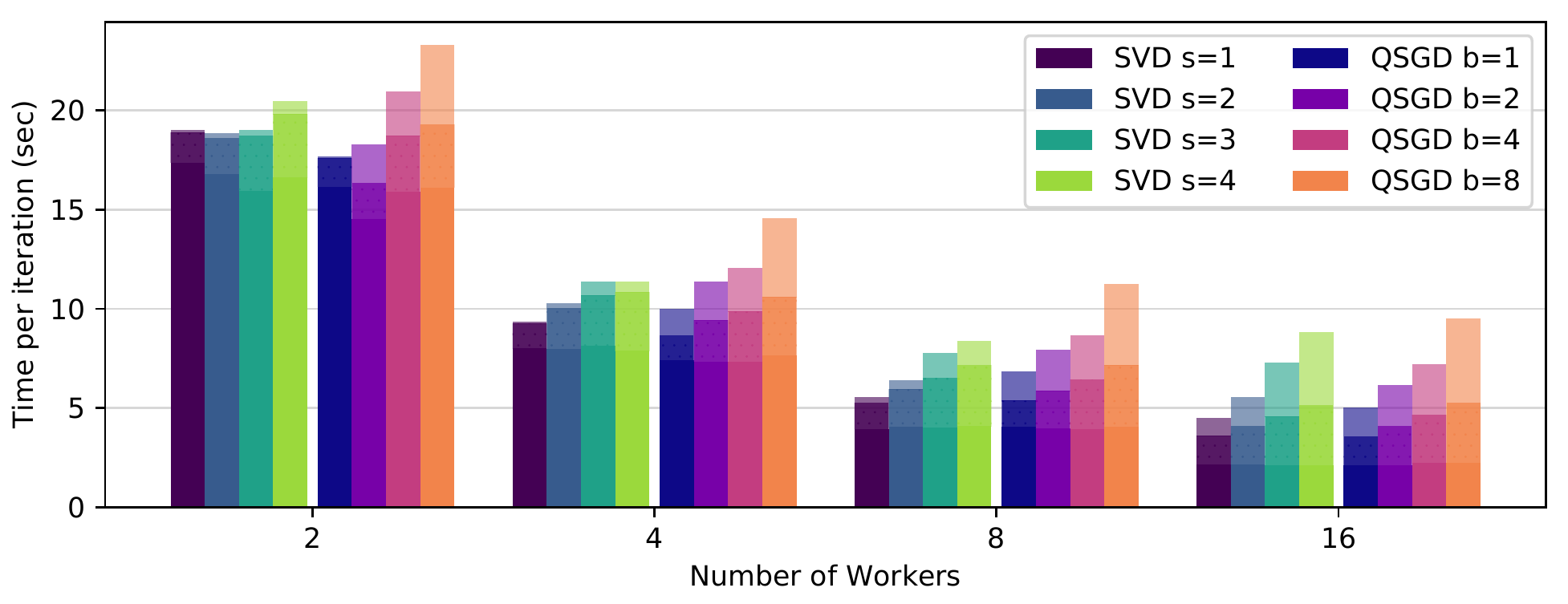}
    \caption{\small The timing of the gradient coding methods (QSGD and spectral-\atomsgd) for different
             quantization levels, $b$ bits and $s$ sparsity budget respectively for each worker when using a ResNet-34 model on CIFAR10. For brevity, we use SVD to denote spectral-\atomsgd. The bars represent the total iteration time and are divided into
             computation time
             (bottom, solid), encoding time (middle, dotted) and communication time (top, faded).}
    \label{scale-performance}
\end{figure}

\paragraph{Scalability} We study the scalability of these sparsification methods on clusters of different sizes. We used clusters with one PS and $n=2, 4, 8, 16$ compute nodes. We ran ResNet-34 on CIFAR-10 using mini-batch SGD with batch size $512$ split among compute nodes. The experiment was run on m5.4xlarge instances of AWS EC2 and the results are shown in Figure \ref{scale-performance}.

While increasing the size of the cluster, decreases the computational cost per worker, it causes the communication overhead to grow. We denote as computational cost, the time cost required by each worker for gradient computations, while the communication overhead is represented by the amount time the PS waits to receive the gradients by the slowest worker. This increase in communication cost is non-negligible, even for moderately-sized networks with sparsified gradients. We observed a trade-off in both sparsification approaches between the information retained in the messages after sparsification and the communication overhead.

\paragraph{End-to-end convergence performance} We evaluate the end-to-end convergence performance on different datasets and neural networks, training with spectral-\atomsgd (with sparsity budget $s= 1, 2 ,3, 4$), QSGD (with $n = 1, 2, 4, 8$ bits), and ordinary mini-batch SGD. The datasets and models are summarized in Table \ref{Tab:DataStat}. We use ResNet-18 \cite{he2016deep} and VGG11-BN \cite{simonyan2014very} for CIFAR-10 \cite{krizhevsky2009learning} and SVHN \cite{netzer2011reading}. Again, for each of these  methods we tune the step size. The experiments were run on a cluster of 16 compute nodes instantiated on g2.2xlarge instances.

The gradients of convolutional layers are 4 dimensional tensors with shape of $[x, y, k, k]$ where $x, y$ are two spatial dimensions and $k$ is the size of the convolutional kernel. However, matrices are required to compute the SVD for spectral-\atomsgd{}, and we choose to reshape each layer into a matrix of size $[x y/2, 2k^2]$. This provides more flexibility on the sparsity budget for the SVD sparsification. For QSGD, we use the bucketing and Elias recursive coding methods proposed in \cite{alistarh2017qsgd}, with bucket size equal to the number of parameters in each layer of the neural network.

\begin{figure*}[ht]
	\centering
	\subfigure[CIFAR-10, ResNet-18, Best of QSGD and SVD]{\includegraphics[width=0.25\linewidth]{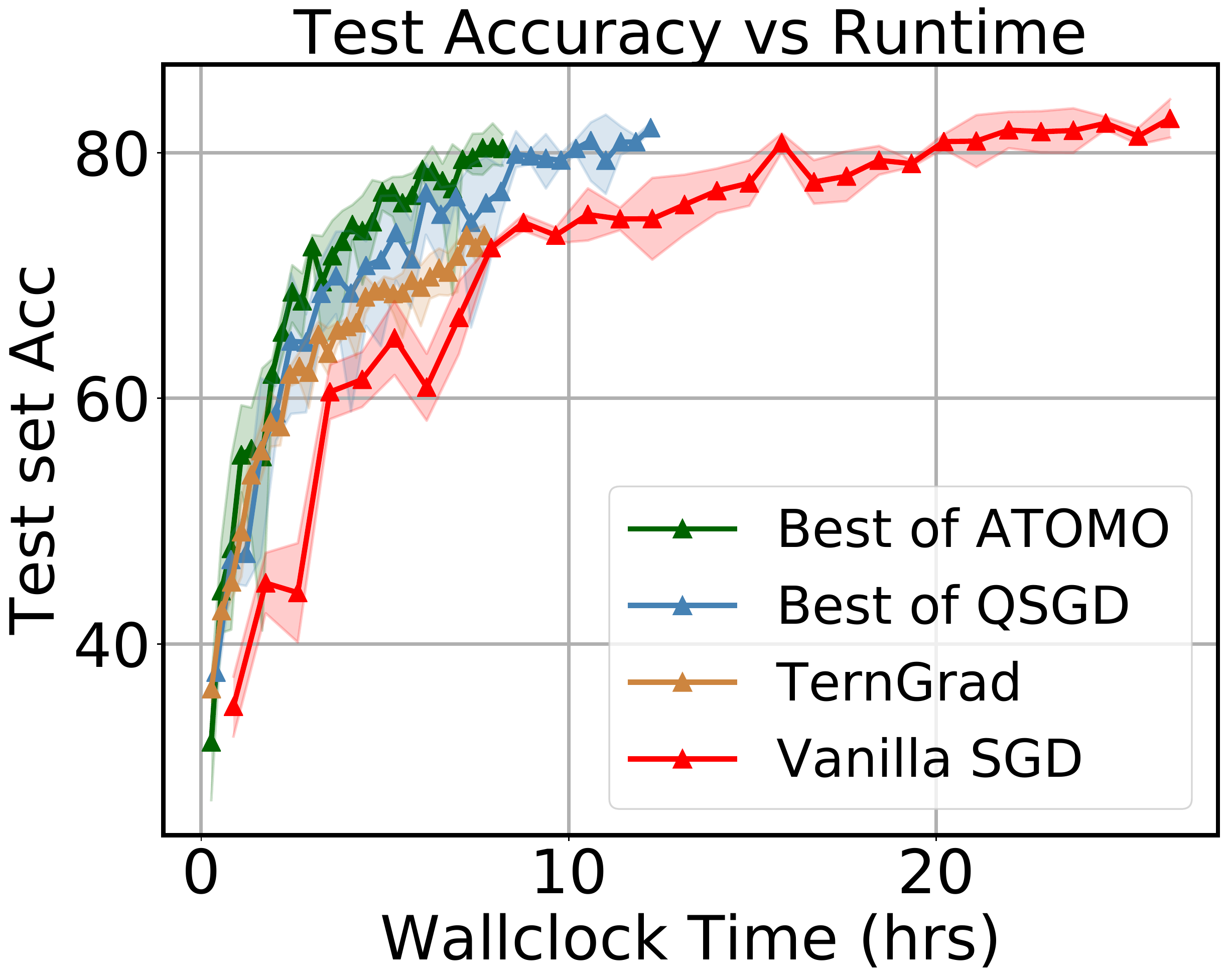}}
	\subfigure[SVHN, ResNet-18, Best of QSGD and SVD]{\includegraphics[width=0.25\linewidth]{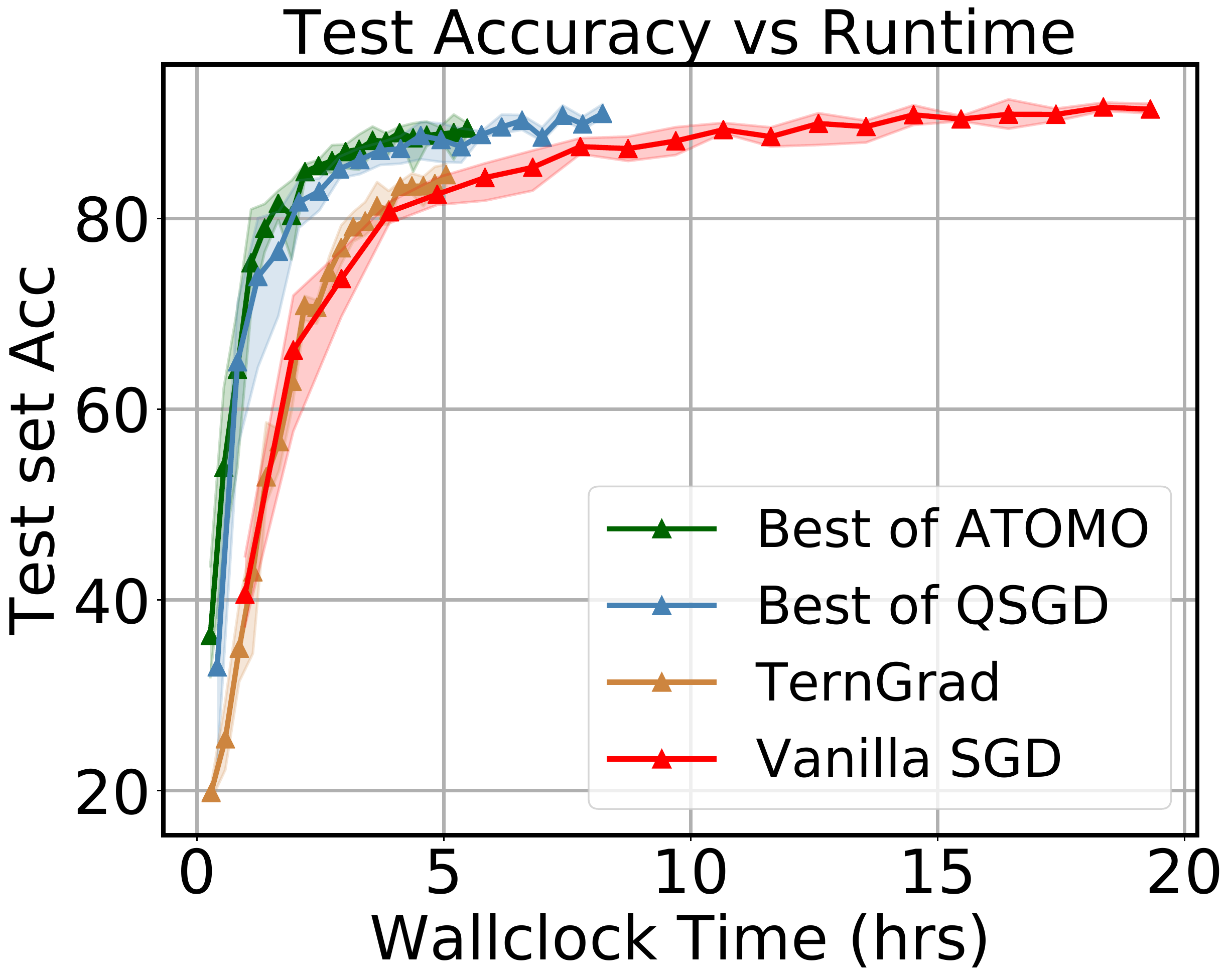}}
	\subfigure[CIFAR-10, VGG11, Best of QSGD and SVD]{\includegraphics[width=0.25\linewidth]{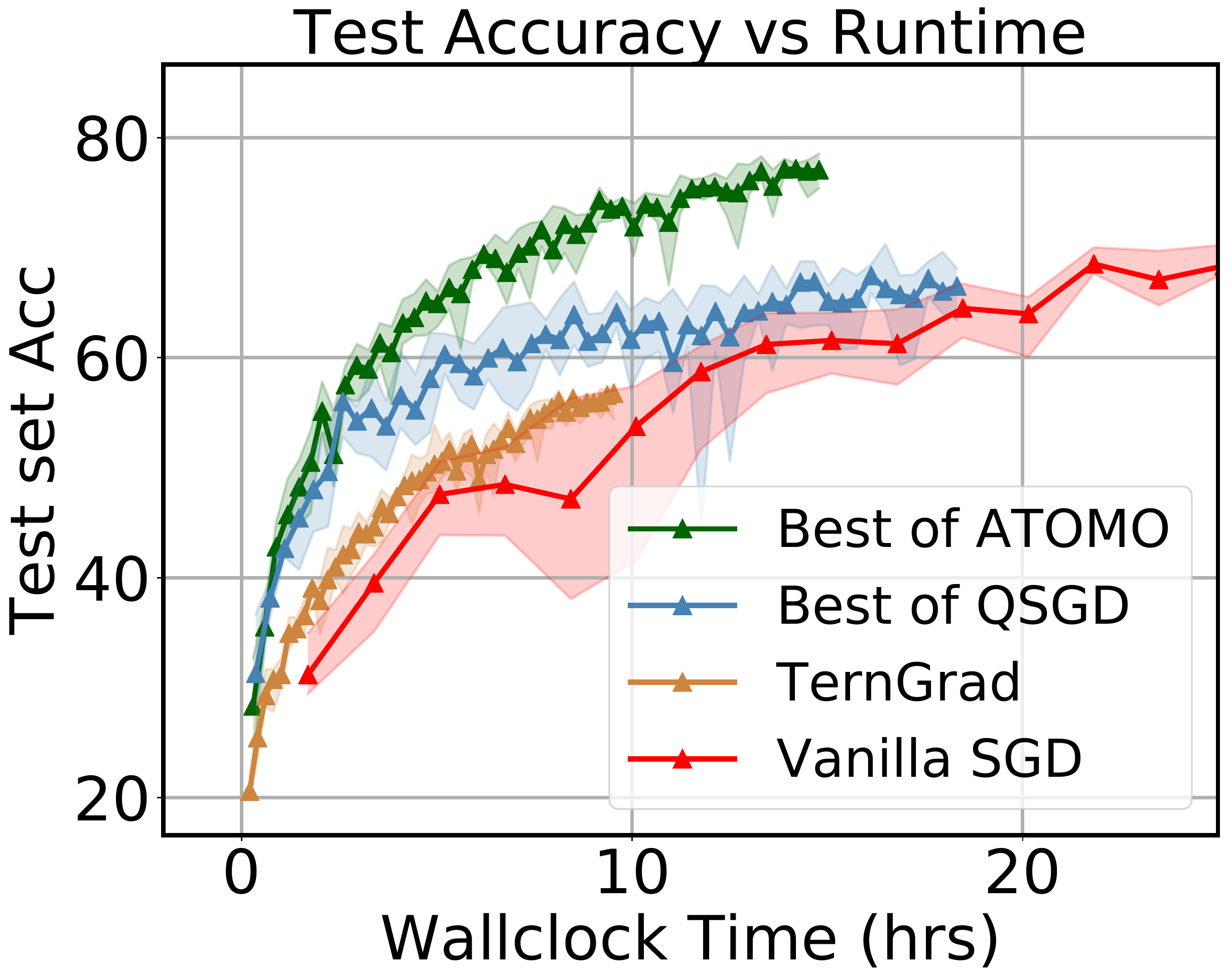}}
	\caption{\small Convergence rates for the best performance of QSGD and spectral-\atomsgd, alongside TernGrad and vanilla SGD. (a) uses ResNet-18 on CIFAR-10, (b) uses ResNet-18 on SVHN, and (c) uses VGG-11-BN on CIFAR-10. For brevity, we use SVD to denote spectral-\atomsgd{}.}
	\label{fig-convergence-runtime}
\end{figure*}

Figure \ref{fig-convergence-runtime} shows how the testing accuracy varies with  wall clock time. Tables~\ref{tab:SpeedupCIFAR10} and \ref{tab:SpeedupSVHN} give a detailed account of speedups of singular value sparsification compared to QSGD. In these tables, each method is run until a specified accuracy.

\begin{table}[ht]
	\centering
    \includegraphics[width=0.35\textwidth]{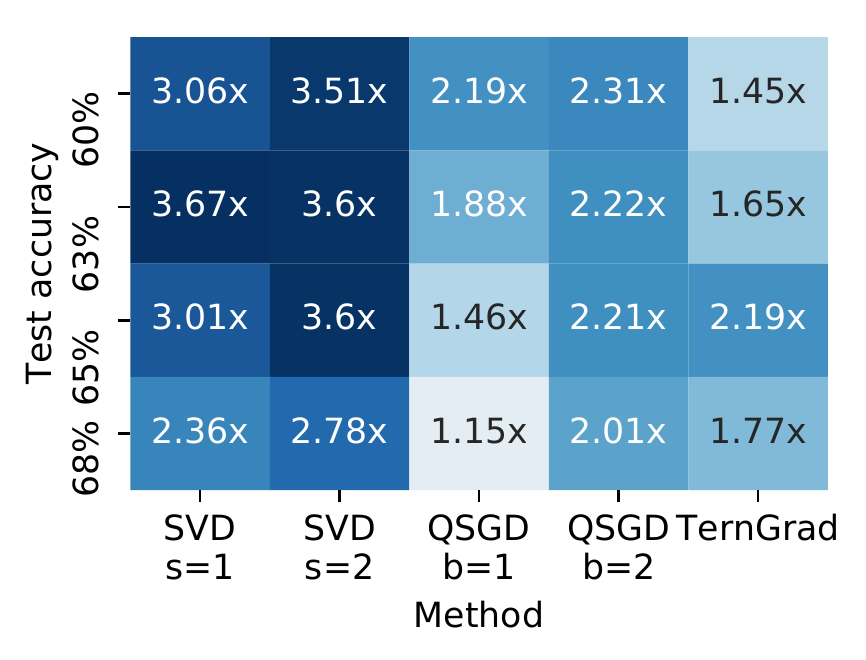}
    \includegraphics[width=0.35\textwidth]{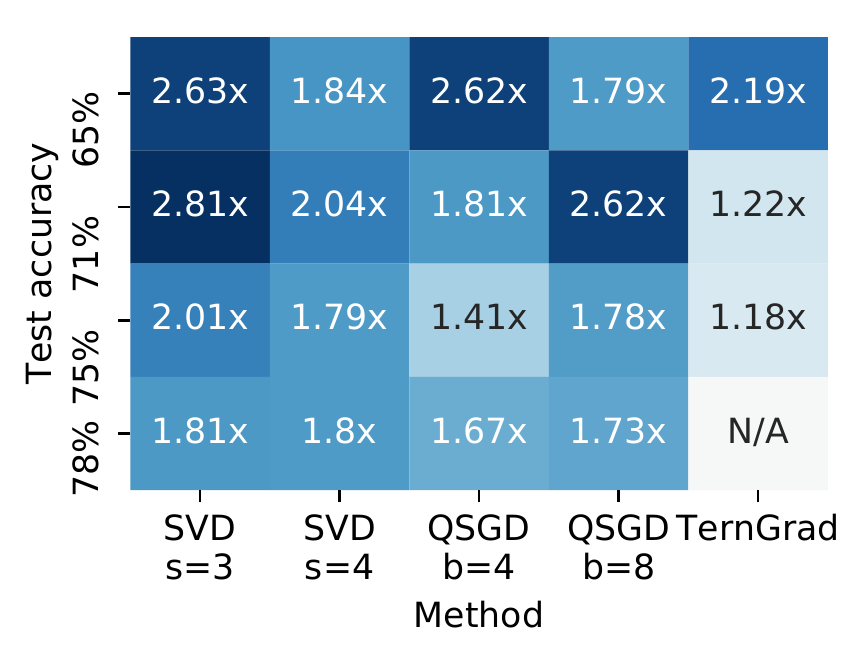}
	\caption{Speedups of spectral-\atomsgd{} with sparsity budget $s$, $b$-bit QSGD, and TernGrad using ResNet-18 on CIFAR10 over vanilla SGD. N/A stands for the method fails to reach a certain Test accuracy in fixed iterations.}
	\label{tab:SpeedupCIFAR10}%
\end{table}%

\begin{table}[ht]
	\centering
	\includegraphics[width=0.35\textwidth]{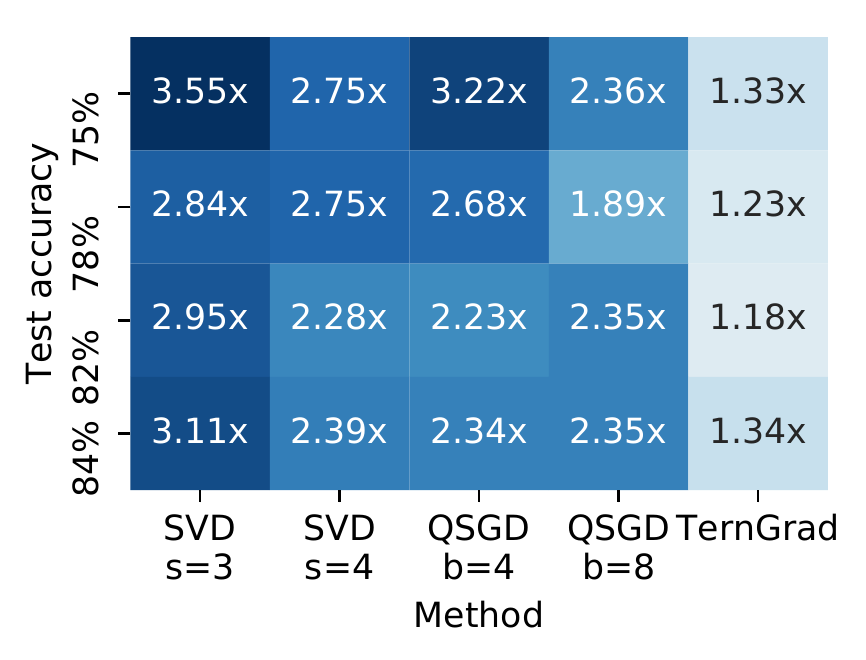}
	\includegraphics[width=0.35\textwidth]{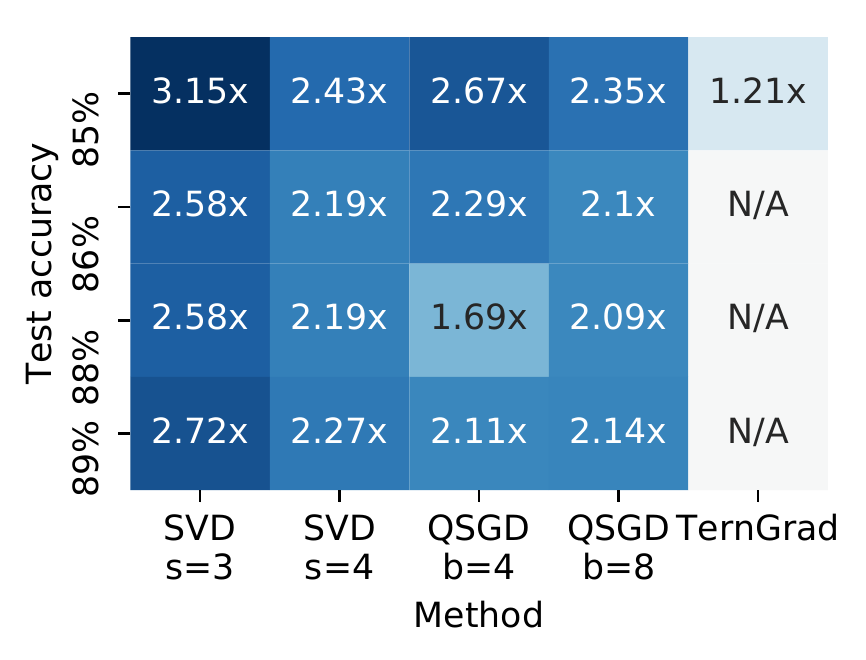}
	\caption{Speedups of spectral-\atomsgd{} with sparsity budget $s$ and $b$-bit QSGD, and TernGrad using ResNet-18 on SVNH over vanilla SGD. N/A stands for the method fails to reach a certain Test accuracy in fixed iterations.}
	\label{tab:SpeedupSVHN}%
\end{table}%

We observe that QSGD and \atomsgd{} speed up model training significantly and achieve similar accuracy to vanilla mini-batch SGD.
We also observe that the best performance is not achieve with the most sparsified, or quantized method, but the optimal method lies somewhere in the middle where enough information is preserved during the sparsification.
For instance, 8-bit QSGD converges faster than 4-bit QSGD, and spectral-\atomsgd{} with sparsity budget  3, or 4 seems to be the fastest. 
Higher sparsity can lead to a faster running time, but extreme sparsification can adversely affect convergence. For example, for a fixed number of iterations, 1-bit QSGD has the smallest time cost, but may converge much more slowly to an accurate model.

\vspace{5cm}

%% file: atomo_conclusion.tex
% !TeX root = atomo_arxiv.tex

\section{Conclusion}\label{sec:conclusion}

	In this paper, we present and analyze \atomsgd{}, a general sparsification method for distributed stochastic gradient based methods. \atomsgd{} applies to any atomic decomposition, including the entry-wise and the SVD of a matrix. \atomsgd{} generalizes 1-bit QSGD and TernGrad, and provably minimizes the variance of the sparsified gradient subject to a sparsity constraint on the atomic decomposition. We focus on the use \atomsgd{} for sparsifying matrices, especially the gradients in neural network training. We show that applying \atomsgd{} to the singular values of these matrices can lead to faster training than both vanilla SGD or QSGD, for the same communication budget. We present extensive experiments showing that \atomsgd{} can lead to up to a $2\times$ speed-up in training time over QSGD and up to $3\times$ speed-up in training time over TernGrad.

	In the future, we plan to explore the use of \atomsgd{} with Fourier decompositions, due to its utility and prevalence in signal processing. More generally, we wish to investigate which atomic sets lead to reduced communication costs. We also plan to examine how we can sparsify and compress gradients in a joint fashion to further reduce communication costs. Finally, when sparsifying the SVD of a matrix, we only sparsify the singular values. We also note that it would be interesting to explore jointly sparsification of the SVD and and its singular vectors, which we leave for future work.

%% file: atomo_appendixA.tex
% !TeX root = atomo_arxiv.tex

\section{Proof of results}

	\subsection{Proof of Lemma \ref{thm:opt_prob2}}

		\begin{proof}
			Suppose we have some $p$ satisfying the conditions in \eqref{eq:opt_prob2}. We define two auxiliary vectors $\alpha, \beta \in \real^n$ by
			\begin{gather*}
			\alpha_i = \dfrac{|\lambda_i|}{\sqrt{p_i}}.\\
			\beta_i = \sqrt{p_i}.\end{gather*}

			Then note that using the fact that $\sum_i p_i = s$, we have
			\begin{align*}
			f(p) = \sum_{i=1}^n \dfrac{\lambda_i^2}{p_i} = \left(\sum_{i=1}^n \dfrac{\lambda_i^2}{p_i}\right)\left(\dfrac{1}{s}\sum_{i=1}^n p_i\right) = \dfrac{1}{s} \left(\sum_{i=1}^n \dfrac{\lambda_i^2}{p_i}\right)\left(\sum_{i=1}^n p_i\right) = \dfrac{1}{s}\|\alpha\|_2^2\|\beta\|_2^2.\end{align*}

			By the Cauchy-Schwarz inequality, this implies
			\begin{equation}\label{eq:c-s}
			f(p) = \dfrac{1}{s}\|\alpha\|_2^2\|\beta\|_2^2 \geq \dfrac{1}{s}|\langle \alpha,\beta\rangle|^2 =\dfrac{1}{s}\left(\sum_{i=1}^n |\lambda_i|\right)^2 = \dfrac{1}{s}\|\lambda\|_1^2.\end{equation}

			This proves the first part of Lemma \ref{thm:opt_prob2}. In order to have $f(p) = \frac{1}{s}\|\lambda\|_1^2$, \eqref{eq:c-s} implies that we need
			$$|\langle \alpha,\beta\rangle| = \|\alpha\|_2\|\beta\|_2.$$
			By the Cauchy-Schwarz inequality, this occurs iff $\alpha$ and $\beta$ are linearly dependent. Therefore, $c\alpha = \beta$ for some constant $c$. Solving, this implies $p_i = c|\lambda_i|$. Since $\sum_{i=1}^n p_i = s$, we have
			$$c\|\lambda\|_1 = \sum_{i=1}^n c|\lambda_i| = \sum_{i=1}^n p_i = s.$$

			Therefore, $c = \|\lambda\|_1/s$, which implies the second part of the theorem.
		\end{proof}

	\subsection{Proof of Lemma \ref{lem:unbalanced}}

		Fix $q$ that is feasible in \eqref{eq:opt_prob1}. To prove Lemma \ref{lem:unbalanced} we will require a lemma. Given the atomic decomposition $g = \sum_{i=1}^n \lambda_ia_i$, we say that $\lambda$ is $s$-unbalanced at $i$ if $|\lambda_i|s > \|\lambda\|_1$, which is equivalent to $g$ being unbalanced in this atomic decomposition at $i$. For notational simplicity, we will assume that $\lambda$ is $s$-unbalanced at $i = 1$. Let $A \subseteq \{2,\ldots, n\}$. We define the following notation:
		\begin{align*}
			s_A &= \sum_{i \in A} q_i.\\
			f_A(q) &= \sum_{i \in A}\dfrac{\lambda_i^2}{q_i}.\\
			(\lambda_A)_i &=   \begin{cases}
				\lambda_i, & \text{for } i \in A, \\
				0, & \text{for } i \notin A.
		    \end{cases}
	    \end{align*}

	    Note that under this notation, Lemma \ref{thm:opt_prob2} implies that for all $p > 0$,
	    \begin{equation}\label{eq:f_S}
	    	f_A(p) \geq \dfrac{1}{s_A}\|\lambda_A\|_1^2.
	    \end{equation}

		\begin{lemma}\label{lem:aux1}
			Suppose that $q$ is feasible and that there is some set $A \subseteq \{2,\ldots,n\}$ such that
			\begin{enumerate}
				\item \label{assume1}$\lambda_A$ is $(s_A+q_1-1)$-balanced.
				\item \label{assume2}$|\lambda_1|(s_A+q_1-1) > \|\lambda_A\|_1$.
			\end{enumerate}
			Then there is a vector $p$ that is feasible satisfying $f(p) \leq f(q)$ and $p_1 = 1$.
		\end{lemma}

		\begin{proof}
			Suppose that such a set $A$ exists. Let $B = \{2,\ldots, n\}\backslash A$. Note that we have
			$$f(q) = \sum_{i =1}^n \dfrac{\lambda_i^2}{q_i} = \dfrac{\lambda_1^2}{q_1} + f_A(q) + f_B(q).$$

			By \eqref{eq:f_S}, this implies
			\begin{equation}\label{eq:q_bound}
			f(q) \geq \dfrac{\lambda_1^2}{q_1} + \dfrac{1}{s_A}\|\lambda_A\|_1^2 + f_B(q).\end{equation}

			Define $p$ as follows.
			$$p_i = \begin{cases}
				1, & \text{for } i = 1, \\
				\dfrac{|\lambda_i|(s_A+q_1-1)}{\|\lambda_A\|_1}, & \text{for } i \in A, \\
				q_i, & \text{for } i \notin A.\end{cases}$$

			Note that by Assumption \ref{assume1} and Lemma \ref{thm:opt_prob2}, we have
			$$f_A(p) = \dfrac{1}{s_A+q_1-1}\|\lambda_A\|_1^2.$$

			Since $p_i = q_i$ for $i \in B$, we have $f_B(p) = f_B(q)$. Therefore, 
			\begin{equation}\label{eq:p_bound}
			f(p) = \lambda_1^2 + \dfrac{1}{s_A+q_1-1}\|\lambda_A\|_1^2 + f_B(q).\end{equation}

			Combining \eqref{eq:q_bound} and \eqref{eq:p_bound}, we have
			\begin{align*}
			f(q)-f(p) &= \lambda_1^2\left(\dfrac{1}{q_1}-1\right)+\|\lambda_A\|_1^2\left(\dfrac{1}{s_A}-\dfrac{1}{s_A+q_1-1}\right)\\
			&= \lambda_1^2\left(\dfrac{1-q_1}{q_1}\right) + \|\lambda_A\|_1^2\left(\dfrac{q_1-1}{s_A(s_A+q_1-1)}\right).\end{align*}

			Combining this with Assumption \ref{assume2}, we have
			\begin{equation}\label{eq:suff_1_aux1}
			f(q)-f(p) \geq \dfrac{\|\lambda_A\|_1^2}{(s_A+q_1-1)^2}\left(\dfrac{1-q_1}{q_1}\right) + \|\lambda_A\|_1^2\left(\dfrac{q_1-1}{s_A(r_A+q_1-1)}\right).\end{equation}

			To show that the RHS of \eqref{eq:suff_1_aux1} is at most $0$, it suffices to show
			\begin{equation}\label{eq:suff_2_aux1}
			s_A \geq q_1(s_A+q_1-1).\end{equation}

			However, note that since $0 < q_1 < 1$, the RHS of \eqref{eq:suff_2_aux1} satisfies
			\begin{align*}
			q_1(s_A+q_1-1) = s_Aq_1 - q_1(1-q_1) \leq s_Aq_1 \leq s_A.\end{align*}

			Therefore, \eqref{eq:suff_2_aux1} holds, completing the proof.
		\end{proof}

		We can now prove Lemma \ref{lem:unbalanced}. In the following, we will refer to Conditions \ref{assume1} and \ref{assume2}, relative to some set $A$, as the conditions required by Lemma \ref{lem:aux1}.

		\begin{proof}
			We first show this in the case that $n = 2$. Here we have the atomic decomposition
			$$g = \lambda_1a_1 + \lambda_2a_2.$$
			The condition that $\lambda$ is $s$-unbalanced at $i = 1$ implies
			$$|\lambda_1|(s-1) > |\lambda_2|.$$

			In particular, this implies $s > 1$. For $A = \{2\}$, Condition \ref{assume1} is equivalent to 
			$$|\lambda_2|(s_A+q_1-2) \leq 0.$$

			Note that $s_A = q_2$ and that $q_1 + q_2 -2= s-2$ by assumption. Since $q_i \leq 1$, we know that $s-2 \leq 0$ and so Condition \ref{assume1} holds. Similarly, Condition \ref{assume2} becomes
			$$|\lambda_1|(s-1) > |\lambda_2|$$
			which holds by assumption. Therefore, Lemma \ref{lem:unbalanced} holds for $n = 2$.

			Now suppose that $n > 2$, $q$ is some feasible probability vector, and that $\lambda$ is $s$-unbalanced at index $1$. We wish to find an $A$ satisfying Conditions \ref{assume1} and \ref{assume2}. Consider $B = \{2,\ldots, n\}$. Note that for such $B$, $s_B+q_1-1 = s-1$. By our unbalanced assumption, we know that Condition \ref{assume2} holds for $B = \{2,\ldots, n\}$. If $\lambda_B$ is $(s-1)$-balanced, then Lemma \ref{lem:aux1} implies that we are done.

			Assume that $\lambda_B$ is not $(s-1)$-balanced. After relabeling, we can assume it is unbalanced at $i = 2$. Let $C = \{3,\ldots, n\}$. Therefore,
			\begin{equation}\label{eq:unbalanced_case}
			|\lambda_2|(s-2) > \|\lambda_C\|_1.\end{equation}

			Combining this with the $s$-unbalanced assumption at $i = 1$, we find
			\begin{align*}
			|\lambda_1| &> \dfrac{\|\lambda_B\|_1}{s-1}\\
			&= \dfrac{|\lambda_2|}{s-1} + \dfrac{\|\lambda_C\|_1}{s-1}\\
			&> \dfrac{\|\lambda_C\|_1}{(s-1)(s-2)} + \dfrac{\|\lambda_C\|_1}{s-1}\\
			&= \dfrac{\|\lambda_C\|_1}{s-2}.\end{align*}

			Therefore,
			\begin{equation}\label{eq:finally}
			|\lambda_1|(s-q_2-1) \geq |\lambda_1|(s-2) > \|\lambda_C\|_1.\end{equation}

			Let $D = \{1,3,4,\ldots, n\} = \{1,\ldots, n\}\backslash \{2\}$. Then note that \eqref{eq:finally} implies that $\lambda_D$ is $(s-q_2)$-unbalanced at $i = 1$. Inductively applying this theorem, this means that we can find a vector $p' \in \real^{|D|}$ such that $p'_1 = 1$ and $f_D(p') \leq f_D(q)$. Moreover, $s_D(p') = s-q_2$. Therefore, if we let $p$ be the vector that equals $p'$ on $D$ and with $p_2 = q_2$, we have
			$$f(p_2) = f_C(p') + \frac{\lambda_2^2}{q_2} \leq f_D(q) + \frac{\lambda_2^2}{q_2} = f(q).$$

			This proves the desired result.
		\end{proof}	

%% file: atomo_kkt_conditions.tex
% !TeX root = atomo_arxiv.tex

\section{Analysis of \atomsgd{} via the KKT Condtions}\label{sec:appendix_kkt}

	In this section we show how to derive Algorithm \ref{alg:opt_sparse} using the KKT conditions. Recall that we wish to solve the following optimization problem:
	\begin{equation}\label{eq:atomo_opt}
		\min_{p}~f(p):=\sum_{i=1}^n \dfrac{\lambda_i^2}{p_i}\quad \text{subject to}~~ \forall i,~0 < p_i \leq 1,~~\sum_{i=1}^n p_i = s.
	\end{equation}	

	We first note a few immediate consequences.
	\begin{enumerate}
		\item If $s > n$ then the problem is infeasible. Note that when $s \geq n$, the optimal thing to do is to set all $p_i = 1$, in which case no sparsification takes place.
		\item If $\lambda_i = 0$, then $p_i = 0$. This follows from the fact that this $p_i$ does not change the value of $f(p)$, and the objective could be decreased by allocating more to the $p_j$ associated to non-zero $\lambda_j$. Therefore we can assume that all $\lambda_i \neq 0$.
		\item If $|\lambda_i| \geq |\lambda_j| > 0$, then we can assume $p_i \geq p_j$. Otherwise, suppose $p_j > p_i$ but $|\lambda_i| \geq |\lambda_j|$. Let $p'$ denote the vector with $p_i, p_j$ switched. We then have
		\begin{align*}
		f(p)-f(p') &= \dfrac{\lambda_i^2-\lambda_j^2}{p_i} + \dfrac{\lambda_j^2-\lambda_i^2}{p_j}\\
		&= \lambda_i^2\left(\dfrac{1}{p_i}-\dfrac{1}{p_j}\right) - \lambda_j^2\left(\dfrac{1}{p_i}-\dfrac{1}{p_j}\right)\\
		&\geq 0.\end{align*}
	\end{enumerate}

	We therefore assume $0 < s \leq n$ and $|\lambda_1| \geq |\lambda_2| \geq \ldots \geq |\lambda_n| > 0$. As above we define $\lambda := [\lambda_1,\ldots, \lambda_n]^T$. While the formulation of \eqref{eq:atomo_opt} does not allow direct application of the KKT conditions, since we have a strict inequality of $0 < p_i$, this is fixed with the following lemma.

	\begin{lemma}The minimum of \eqref{eq:atomo_opt} is achieved by some $p^*$ satisfying
	$$p_i^* \geq \dfrac{s\lambda_i^2}{n\|\lambda\|_2^2}.$$\end{lemma}

	\begin{proof}Define $\overline{p}$ by $\overline{p}_i = s/n$. This vector is clearly feasible in \eqref{eq:atomo_opt}. Let $p$ be any feasible vector. If $f(p) \leq f(q)$ then for any $i \in [n]$ we have
	$$\dfrac{\lambda_i^2}{p_i} \leq f(p) \leq f(\overline{p}).$$
	Therefore, $p_i \geq \lambda_i^2/f(\overline{p})$. A straightforward computations shows that $f(\overline{p}) = n\|\lambda\|_2^2/s$. Note that this implies that we can restrict to the feasbile set
	$$\dfrac{s\lambda_i^2}{n\|\lambda\|_2^2} \leq p_i \leq 1.$$
	This defines a compact region $C$. Since $f$ is continuous on this set, its maximum value is obtained at some $p^*$.\end{proof}

	The KKT conditions then imply that at any point $p$ solving \eqref{eq:atomo_opt}, we must have
	\begin{equation}\label{eq:kkt1}
	0 \leq 1-p_i \perp \mu-\dfrac{\lambda_i^2}{p_i} \geq 0,~~i=1,2,\ldots,n\end{equation}
	for some $\mu \in \real$. Since $|\lambda_i| > 0$ for all $i$, we actually must have $\mu > 0$.	We therefore have two conditions for all $i$.
	\begin{enumerate}
		\item $p_i = 1 \implies \mu \geq \lambda_i^2$.
		\item $p_i < 1 \implies p_i = |\lambda_i|/\sqrt{\mu}$.
	\end{enumerate}
	Note that in either case, to have $p_1$ feasible we must have $\mu \geq \lambda_1^2$. Combining this with the fact that we can always select $p_1 \geq p_2 \geq\ldots \geq p_n$, we obtain the following partial characterization of the solution to \eqref{eq:atomo_opt}. For some $n_s \in [n]$, we have $p_1,\ldots, p_{n_s} = 1$ while $p_i = |\lambda_i|/\sqrt{\mu} \in (0,1)$ for $i = n_s+1,\ldots, n$.
	Combining this with the constraint that $\sum_{i=1}^n p_i = s$, we have
	\begin{equation}
	s = \sum_{i=1}^n p_i = n_s + \sum_{i=n_s+1}^n p_i = n_s + \sum_{i=n_s+1}^n \dfrac{|\lambda_i|}{\sqrt{\mu}}.\end{equation}
	Rearranging, we obtain
	\begin{equation}
	\mu = \dfrac{\left(\sum_{i=n_s+1}^n |\lambda_i|\right)^2}{(s-n_s)^2}\end{equation}
	which then implies that
	\begin{equation}\label{eq:p_i}
	p_i = 1,~~i = 1,\ldots,n_s,~~~p_i = \dfrac{|\lambda_i|(s-n_s)}{\sum_{j=n_s+1}^n |\lambda_j|},~~i = n_s+1,\ldots, n.\end{equation}

	Thus, we need to select $n_s$ such that the $p_i$ in \eqref{eq:p_i} are bounded above by 1. Let $n_s^*$ denote the first element of $[n]$ for which this holds. Then the condition that $p_i \leq 1$ for $i = n_s^*+1,\ldots, n$ is exactly the condition that $[\lambda_{n_s^*+1},\ldots,\lambda_n]$ is $(s-n_s)$-balanced (see Definition \ref{def:s-balanced}. In particular, Lemma \ref{thm:opt_prob2} implies that, fixing $p_i = 1$ for $i = 1,\ldots, n_s^*$, the optimal way to assign the remaining $p_i$ is by
	$$p_i = \dfrac{|\lambda_i|(s-n_s^*)}{\sum_{j=n_s^*+1}^n |\lambda_j|}.$$
	This agrees with \eqref{eq:p_i} for $n_s = n_s^*$. In particular, the minimal value of $f$ occurs at the first value of $n_s$ such that the $p_i$ in \eqref{eq:p_i} are bounded above by 1.

	Algorithm \ref{alg:opt_sparse} scans through the sorted $\lambda_i$ and finds the first value of $n_s$ for which the probabilities in \eqref{eq:p_i} are in $[0,1]$, and therefore finds the optimal $p$ for \eqref{eq:atomo_opt}. The runtime is dominated by the $O(n\log n)$ sorting cost. It is worth noting that we could perform the algorithm in $O(sn)$ time as well. Instead of sorting and then iterating through the $\lambda_i$ in order, at each step we could simply select the next largest $|\lambda_i|$ not yet seen and perform an analogous test and update as in the above algorithm. Since we would have to do the selection step at most $s$ times, this leads to an $O(sn)$ complexity algorithm.

%% file: atomo_appendixB.tex
% !TeX root = atomo_arxiv.tex

\section{Equivalence of norms}\label{sec:matrix_norms}

    We are often interested in comparing norms on vectors spaces. This naturally leads to the following definition.

    \begin{definition}Let $V$ be a vector space over $\real$ or $\cmp$. Two norms $\|\cdot\|_a, \|\cdot\|_b$ are equivalent if there are positive constants $C_1, C_2$ such that
    $$C_1 \|x\|_a \leq \|x\|_b \leq C_2\|x\|_a$$
    for all $x \in V$.\end{definition}

    As it turns out, norms on finite-dimensional vector spaces are always equivalent.

    \begin{theorem}Let $V$ be a finite-dimensional vector space over $\real$ or $\cmp$. Then all norms are equivalent.\end{theorem}

    In order to compare norms, we often wish to determine the tightest constants which give equivalence between them. In Section \ref{sec:compare_matrix_sparse}, we are particularly interested in comparing the $\|X\|_*$ and $\|X\|_{1,1}$ on the space of $n\times m$ matrices. We have the following lemma.

    \begin{lemma}For all $n\times m$ real matrices, 
    $$\frac{1}{\sqrt{nm}}\|X\|_{1,1} \leq \|X\|_* \leq \|X\|_{1,1}.$$\end{lemma}

    \begin{proof}
    Suppose that $X$ has the singular value decomposition
    $$X = \sum_{i=1}^r \sigma_iu_iv_i^T.$$

    We will first show the left inequality. First, note that for any $n\times m$ matrix $A$, $\|A\|_{1,1}\leq \sqrt{nm}\|A\|_F$. This follows directly from the fact that for a $n$-dimensional vector $v$, $\|v\|_1 \leq \sqrt{n}\|v\|_2$. We will also use the fact that for any vectors $u \in \real^n, v \in \real^m$, $\|uv^T\|_F = \|u\|_2\|v\|_2$. We then have
    \begin{align*}
    \|X\|_{1,1} &= \left\|\sum_{i=1}^r \sigma_iu_iv_i^T\right\|_{1,1}\\
    &\leq \sum_{i=1}^r \sigma_i \|u_iv_i^T\|_{1,1}\\
    & = \sum_{i=1}^r \sigma_i\sqrt{nm}\|u_iv_i^T\|_F\\
    &= \sum_{i=1}^r \sigma_i \sqrt{nm}\|u_i\|_2\|v_i\|_2\\
    &= \|X\|_*.\end{align*}

    For the right inequality, note that we have
    $$X = \sum_{i,j}X_{i,j}e_ie_j^T$$
    where $e_i \in \real^n$ is the $i$-th standard basis vector, while $e_j \in \real^m$ is the $j$-th standard basis vector. We then have
    $$\|X\|_* \leq \sum_{i,j}|X_{i,j}|\|e_ie_j^T\|_* = \sum_{i,j}|X_{i,j}| = \|X\|_{1,1}.$$\end{proof}

    In fact, these are the best constants possible. To see this, first consider the matrix $X$ with a 1 in the upper-left entry and 0 elsewhere. Clearly, $\|X\|_* = \|X\|_{1,1} = 1$, so the right-hand inequality is tight. For the left-hand inequality, consider the all-ones matrix $X$. This has one singular value, $\sqrt{nm}$, so $\|X\|_{*} = \sqrt{nm}$. On the other hand, $\|X\|_{1,1} = nm$. Therefore, $\|X\|_{1,1} = \sqrt{nm}\|X\|_*$ in this case.

%% file: atomo_appendixC.tex
\section{Hyperparameter optimization}
\label{sec:hyperparam}
We firstly provide results of step size tunning, as shows in Table \ref{tab:stepsize} we reported stepsize tunning results for all of our experiments. We tuned these step sizes by evaluating many logarithmically spaced step sizes (e.g., $2^{-10},\ldots,2^0$) and evaluated on validation loss.

This step sizes tuning, for 8 gradient coding methods and 3 datasets was only possible because fairly small networks were used.

\begin{table}[th]
	\centering
	\caption{Tuned stepsizes for experiments}
	%	\vspace{0.3 cm}
	\begin{tabular}{|c|c|c|c|}
		\hline Experiments
        & CIFAR-10 \& ResNet-18 & SVHN \& ResNet-18  & CIFAR-10 \& VGG-11-BN \bigstrut\\
		\hline
		\hline
		SVD rank 1 & 0.0625 & 0.1 & 0.125 \bigstrut\\
		\hline
		SVD rank 2 & 0.0625 & 0.125 & 0.125  \bigstrut\\
		\hline
		SVD rank 3 & 0.125 & 0.125 & 0.0625  \bigstrut\\
		\hline
		SVD rank 4 & 0.0625 & 0.125 & 0.15 \bigstrut\\
		\hline
		QSGD 1bit & 0.0078125 & 0.0078125 & 0.0009765625 \bigstrut\\
		\hline
		QSGD 2bit & 0.0078125 & 0.0078125 & 0.0009765625  \bigstrut\\
		\hline
		QSGD 4bit & 0.125 & 0.046875 & 0.015625  \bigstrut\\
		\hline
		QSGD 8bit & 0.125 & 0.125 & 0.0625  \bigstrut\\
		\hline
	\end{tabular}%
	\label{tab:stepsize}%
	%	\vspace{-1cm}
\end{table}%

\section{Additional Experiments}
\textbf{Runtime analysis:} We empirically study runtime costs of spectral-\atomsgd{} with sparsity budget set at 1, 2, 3, 6 and made comparisons among $b$-bit QSGD and TernGrad. We deployed distributed training on ResNet-18 with batch size $B=256$ on the CIFAR-10 dataset run with m5.2xlarge instances. As shown in Figure \ref{runtime-analysis}, there is a trade-off between the amount of communication per iteration and the running time for both singular value sparsification and QSGD. In some scenarios, spectral-\atomsgd{} attains a higher compression ratio than QSGD and TernGrad. For example, singular value sparsification with sparsity budget 1 may communicate smaller messages than $\{2,4\}$-bit QSGD and Terngrad.

\begin{figure}[h]
    \centering
    \includegraphics[width=0.95\linewidth]{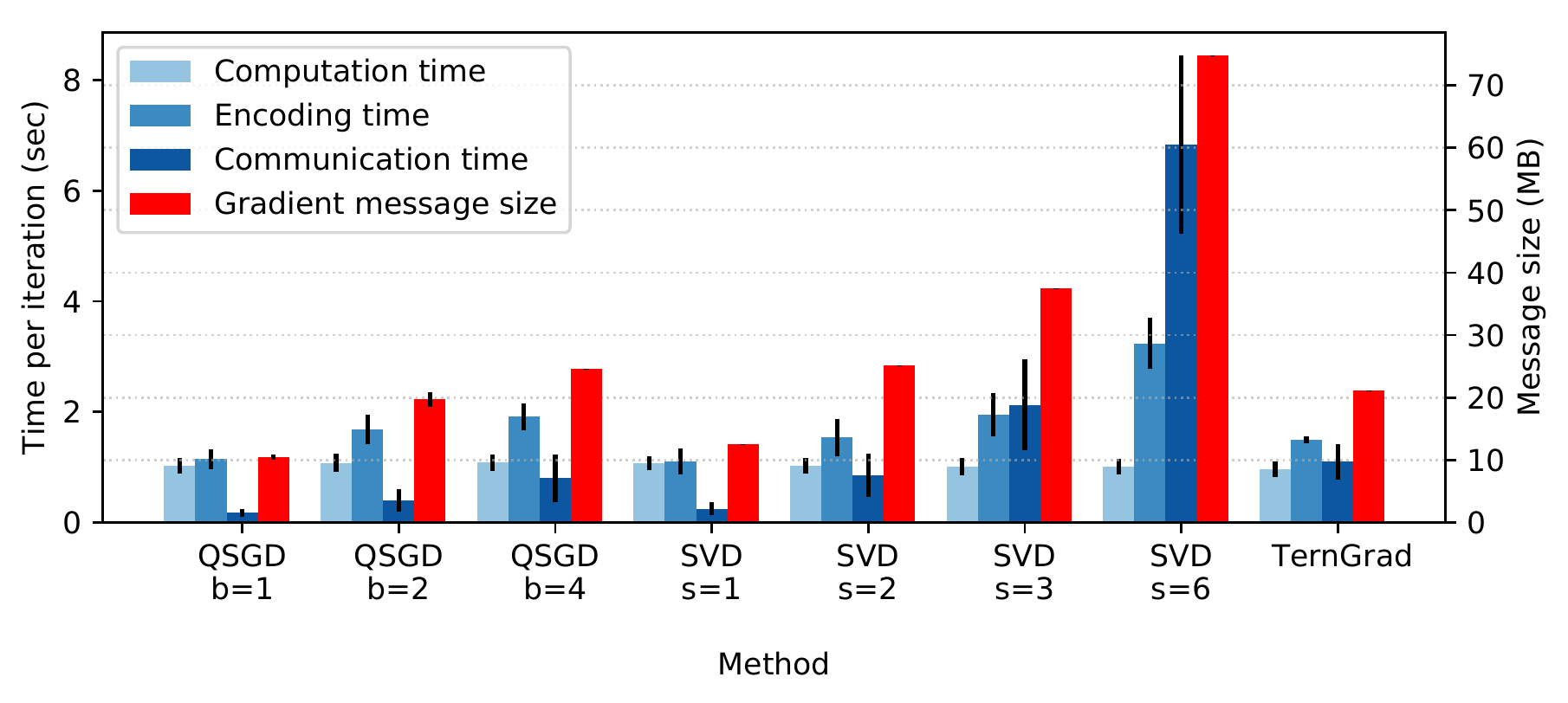}
    \caption{Runtime analysis of different sparsification methods (singular
    value sparsification, QSGD, and TernGrad) for ResNet-18 trained on CIFAR-10.
    The values shown are computation, encoding and communication time as well
    as the size of the message required to send gradients between workers.}
    \label{runtime-analysis}
\end{figure}

\begin{figure*}[h]%[htp]
	\centering
	\subfigure[CIFAR-10, ResNet-18, Best of QSGD and SVD]{\includegraphics[width=0.32\linewidth]{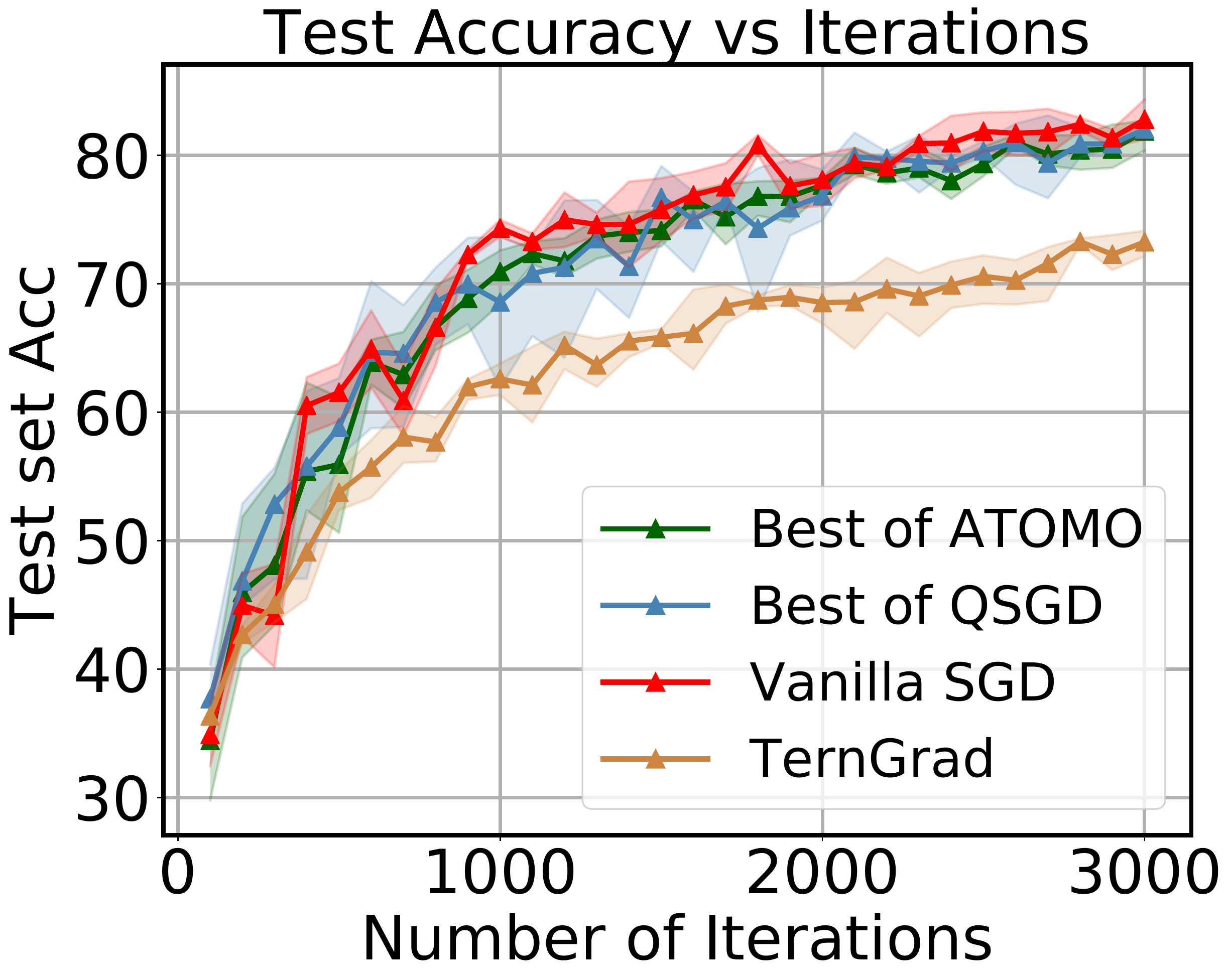}}
	\subfigure[SVHN, ResNet-18, Best of QSGD and SVD]{\includegraphics[width=0.32\linewidth]{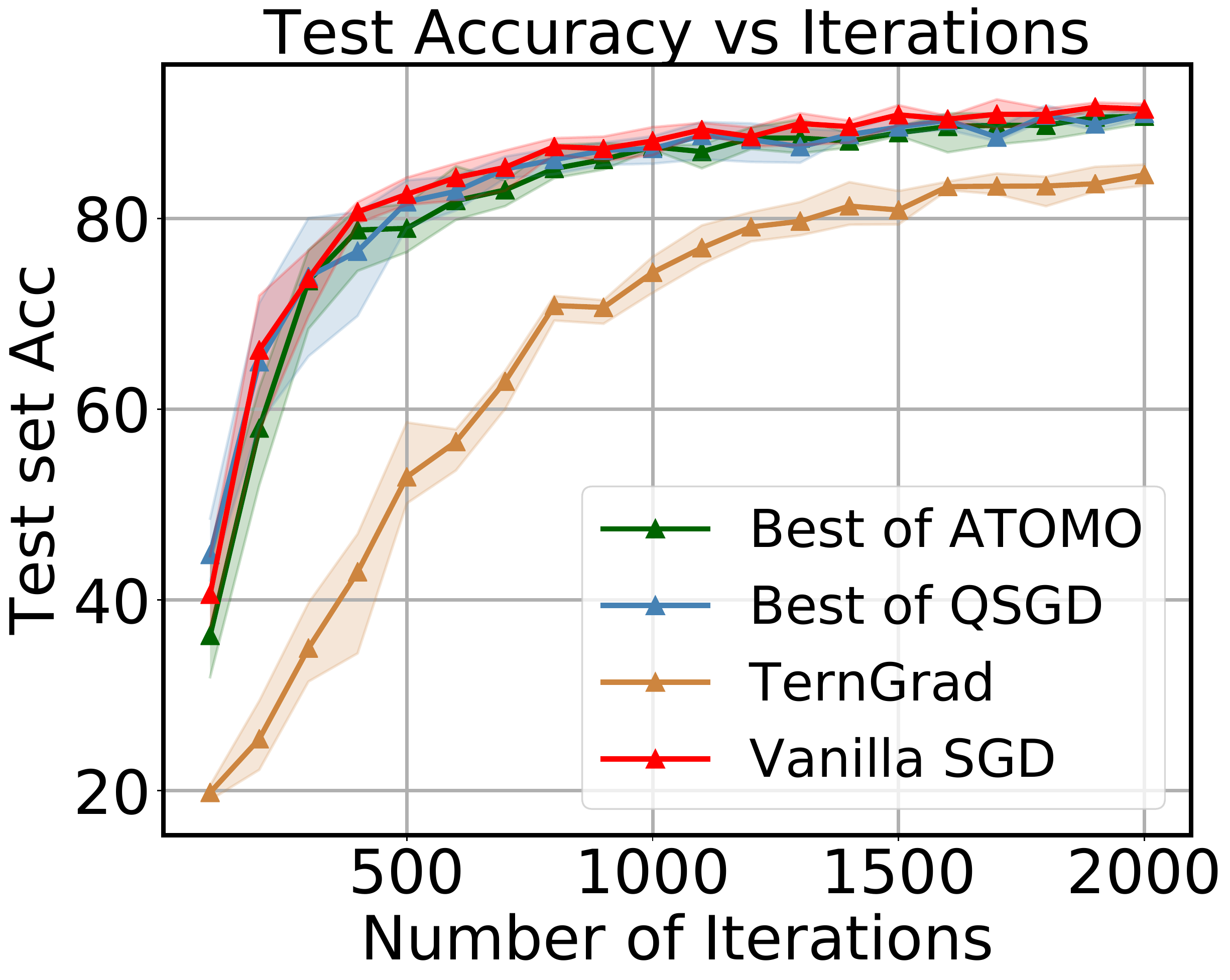}}
	\subfigure[CIFAR-10, VGG11, Best of QSGD and SVD]{\includegraphics[width=0.32\linewidth]{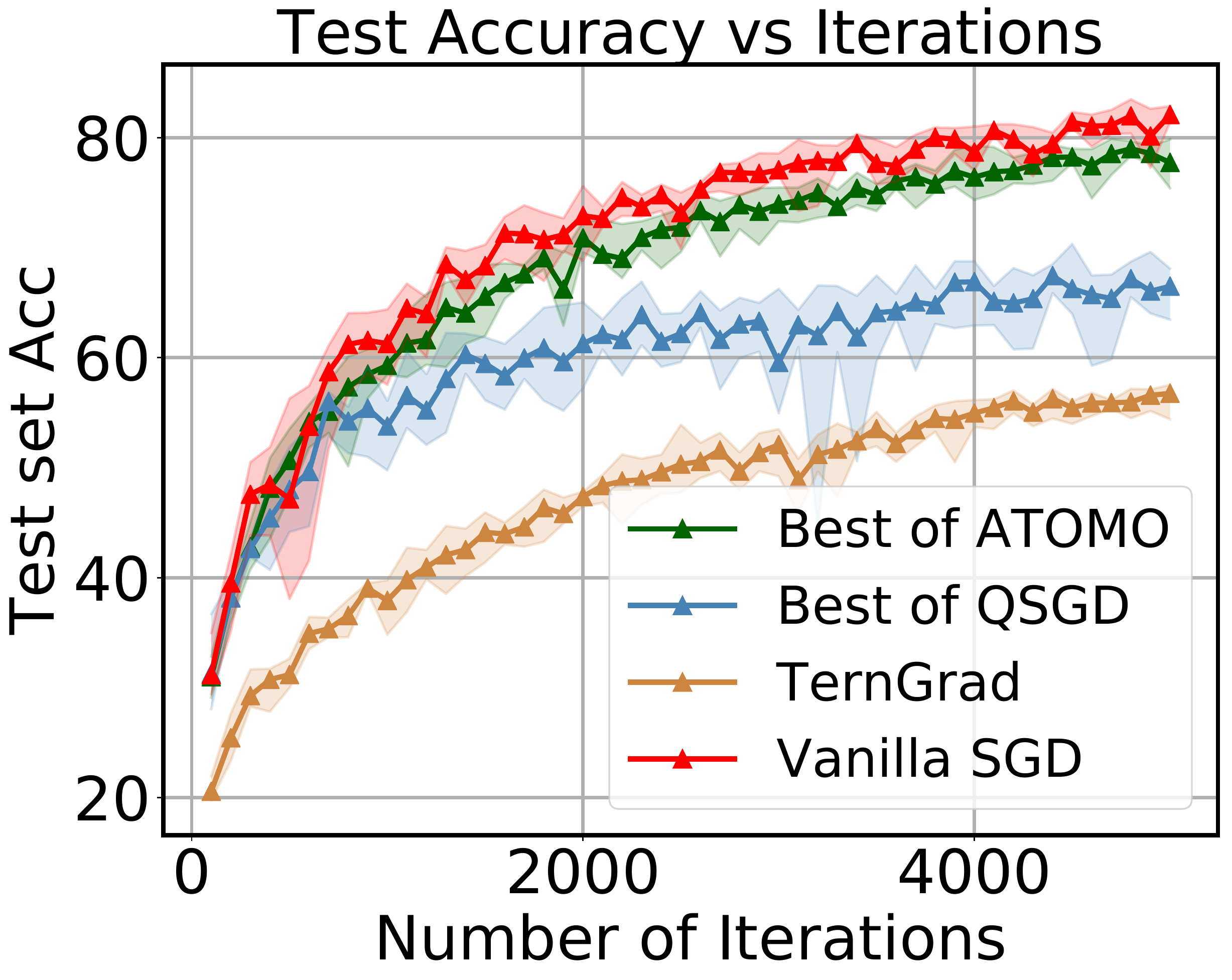}}
	\caption{Convergence rates with respect to number of iterations on: (a) CIFAR-10 on ResNet-18 of best performances from QSGD and SVD (b) SVHN on ResNet-18 of best performances from QSGD and SVD, (c) CIFAR-10 on VGG-11-BN best of performances from  QSGD and SVD}
	\label{fig-convergence-iter}
\end{figure*}

\begin{table}%[htbp]
	\centering
    \includegraphics[width=0.80\textwidth]{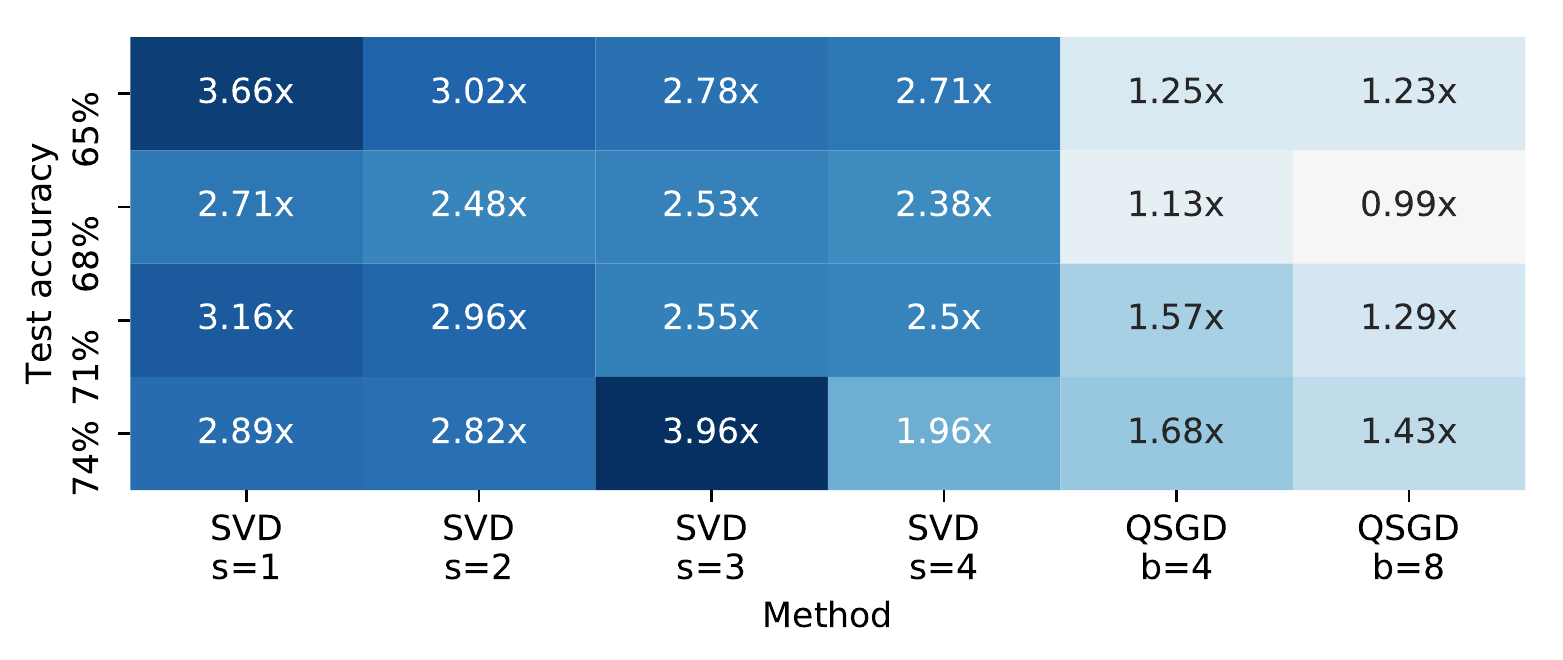}
	\caption{Speedups of spectral-\atomsgd{} with sparsity budget $s$, $b$-bit QSGD, and TernGrad using VGG11 on CIFAR-10 over vanilla SGD.}
	\label{tab:SpeedupVGG}%
\end{table}%